\documentclass[journal]{IEEEtran} 

\IEEEoverridecommandlockouts                          
\makeatletter

\let\proof\@undefined
\let\endproof\@undefined
\makeatother

\usepackage{amsthm}

\newtheorem{assumption*}{Assumption}
\newtheorem{theorem*}{Theorem}
\newtheorem{lemma}{Lemma}
\newtheorem{problem}{Problem}
\newtheorem{proposition}{Proposition}
\newtheorem{observation*}{Observation}
\newtheorem{remark*}{Remark}

\usepackage{mathtools, stackengine}

\let\labelindent\relax
\usepackage[shortlabels]{enumitem}

\usepackage[small]{caption}
\usepackage{subcaption}
\usepackage{url}
\usepackage{algorithm}
\usepackage[noend]{algpseudocode}

\usepackage{pgfplots}
\pgfplotsset{compat=1.7}

\usepackage{booktabs}
\usepackage{makecell}

\usepackage[utf8]{inputenc}

\usepackage{amsmath}
\usepackage{amsfonts}

\usepackage{cite}

\newcommand{\docupdate}[1]{{#1}}

\begin{document}

\title{Anytime Replanning of Robot Coverage Paths for Partially Unknown Environments}
\urldef{\ramesh}\url{m5ramesh@uwaterloo.ca}
\urldef{\smith}\url{stephen.smith@uwaterloo.ca}
\urldef{\fidan}\url{fidan@uwaterloo.ca}
\urldef{\imeson}\url{frank.imeson@avidbots.com}

\author{Megnath Ramesh, Frank Imeson, Baris Fidan, and Stephen L. Smith
\thanks{M. Ramesh (\ramesh) and S. L. Smith (\smith) are with the Department of Electrical and Computer Engineering and B. Fidan (\fidan) is with the Department of Mechanical and Mechatronics Engineering, at the University of Waterloo, Waterloo ON, Canada}
\thanks{F. Imeson (\imeson) is with Avidbots Corp., Kitchener ON, Canada}
}

\maketitle



\begin{abstract}
In this paper, we propose a method to replan coverage paths for a robot operating in an environment with initially unknown static obstacles. Existing coverage approaches reduce coverage time by covering along the minimum number of coverage lines (straight-line paths). However, recomputing such paths online can be computationally expensive resulting in robot stoppages that increase coverage time. A naive alternative is \textit{greedy detour} replanning, i.e., replanning with minimum deviation from the initial path, which is efficient to compute but may result in unnecessary detours. In this work, we propose an anytime coverage replanning approach named \textit{OARP-Replan} that performs near-optimal replans to an interrupted coverage path within a given time budget. We do this by solving linear relaxations of \docupdate{integer linear programs (ILPs)} to identify sections of the interrupted path that can be optimally replanned within the time budget. \docupdate{We validate OARP-Replan in simulation and perform comparisons against a greedy detour replanner and other state-of-the-art coverage planners. We also demonstrate OARP-Replan in experiments using an industrial-level autonomous robot.}
\end{abstract}


\section{Introduction}



Coverage planning problems show up in many real-world applications such as cleaning \cite{hofnerPathPlanningGuidance1995}, manufacturing \cite{kapanogluPatternbasedGeneticAlgorithm2012}, and agriculture \cite{hameedIntelligentCoveragePath2014}. Informally, the coverage planning problem is to find a path that has the robot covering its entire environment with an onboard sensor or tool and minimizes a cost function \cite{galceranSurveyCoveragePath2013}. In this paper, we tackle the problem of replanning optimal coverage paths that have been interrupted by previously unknown static obstacles. We refer to this as the \emph{coverage replanning} problem. This problem arises when businesses such as retail stores and warehouses change their environments frequently to accommodate new products and remapping beforehand each time would be too costly.

\begin{figure*}
\centering
    \subcaptionbox{}
    {\centering
    \includegraphics[width=0.255\linewidth]{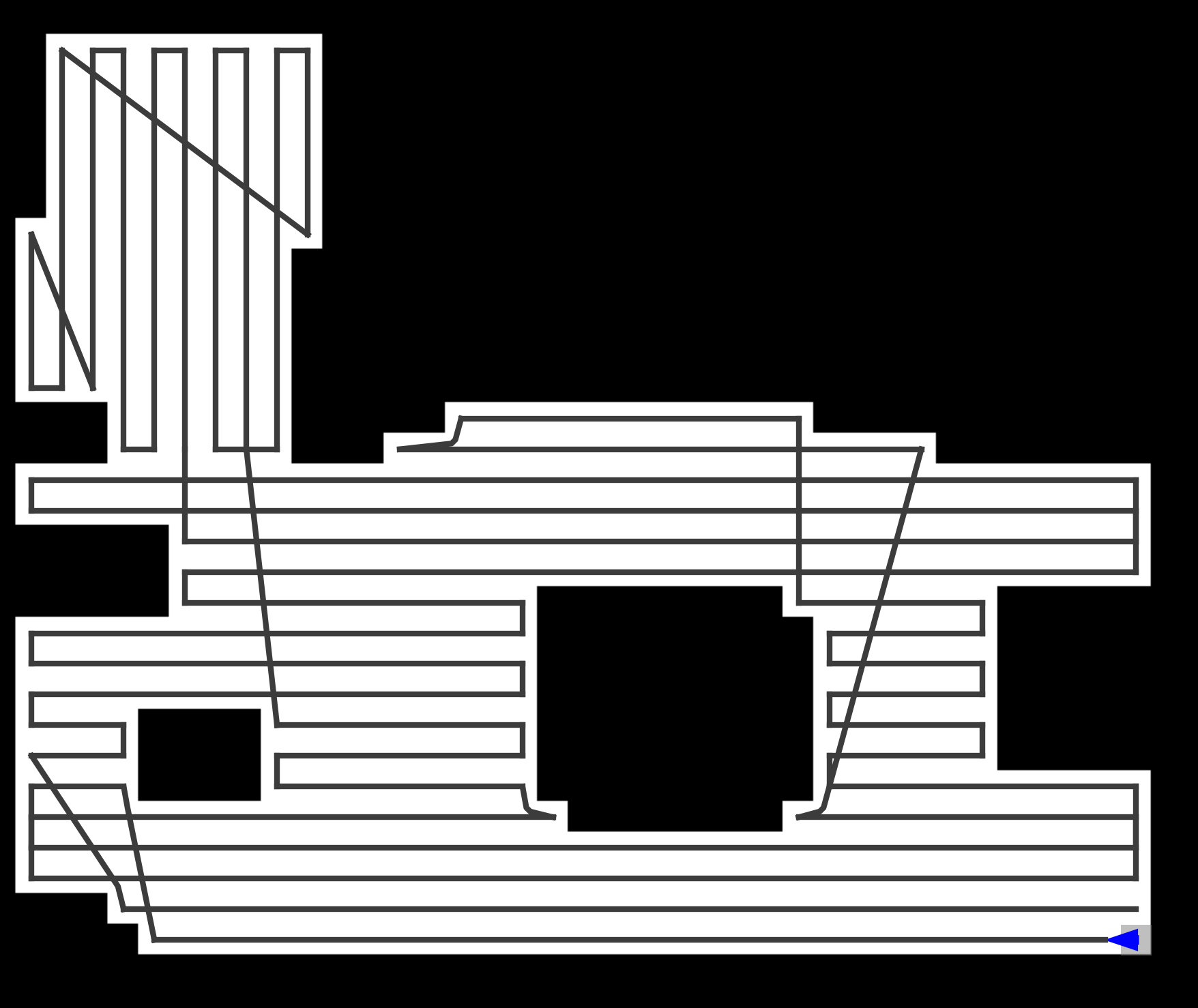}
    }
    \hspace{1cm}
    \subcaptionbox{}
    {\centering
    \includegraphics[width=0.255\linewidth]{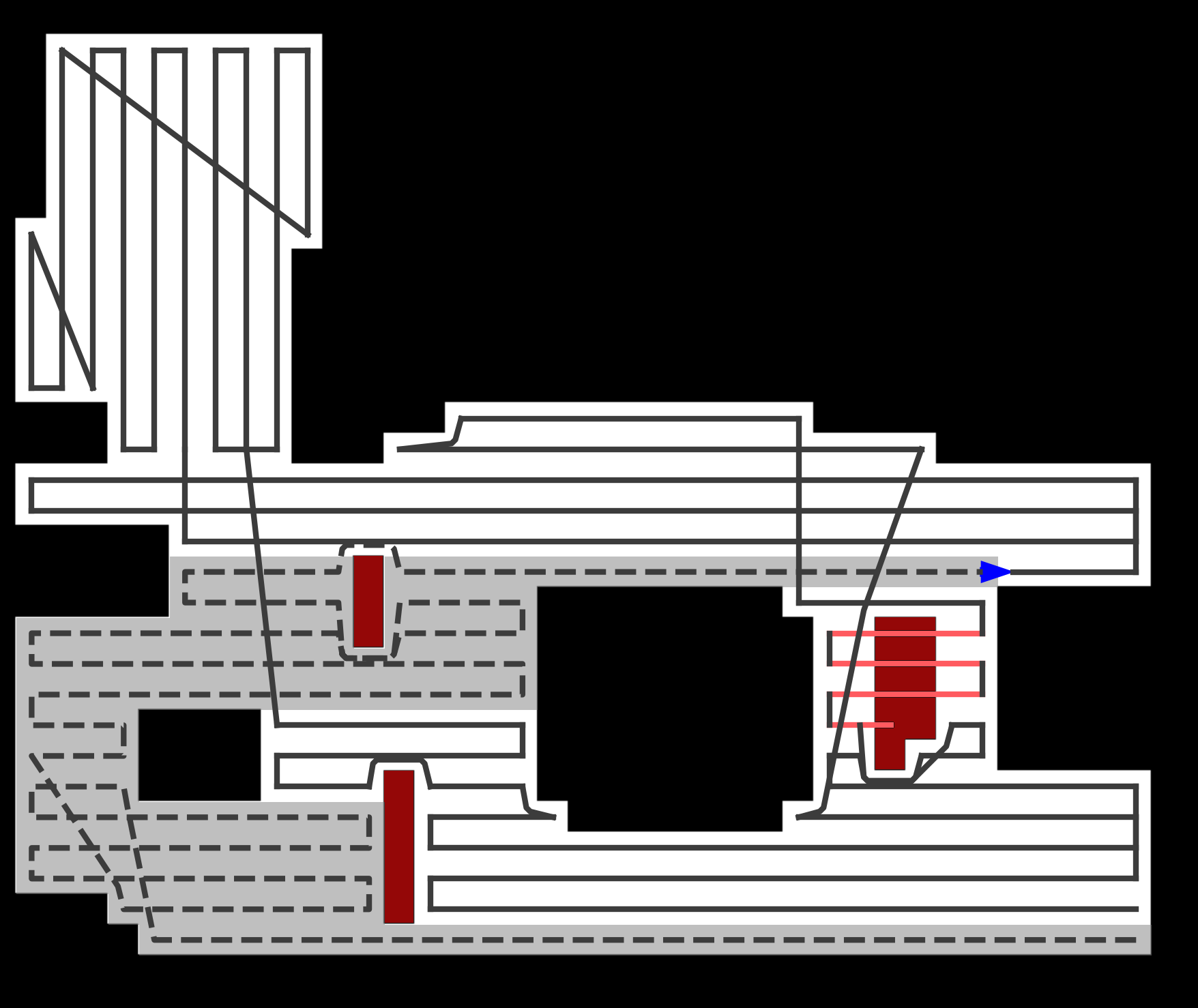}
    }
    \hspace{1cm}
    \subcaptionbox{}
    {\centering
    \includegraphics[width=0.255\linewidth]{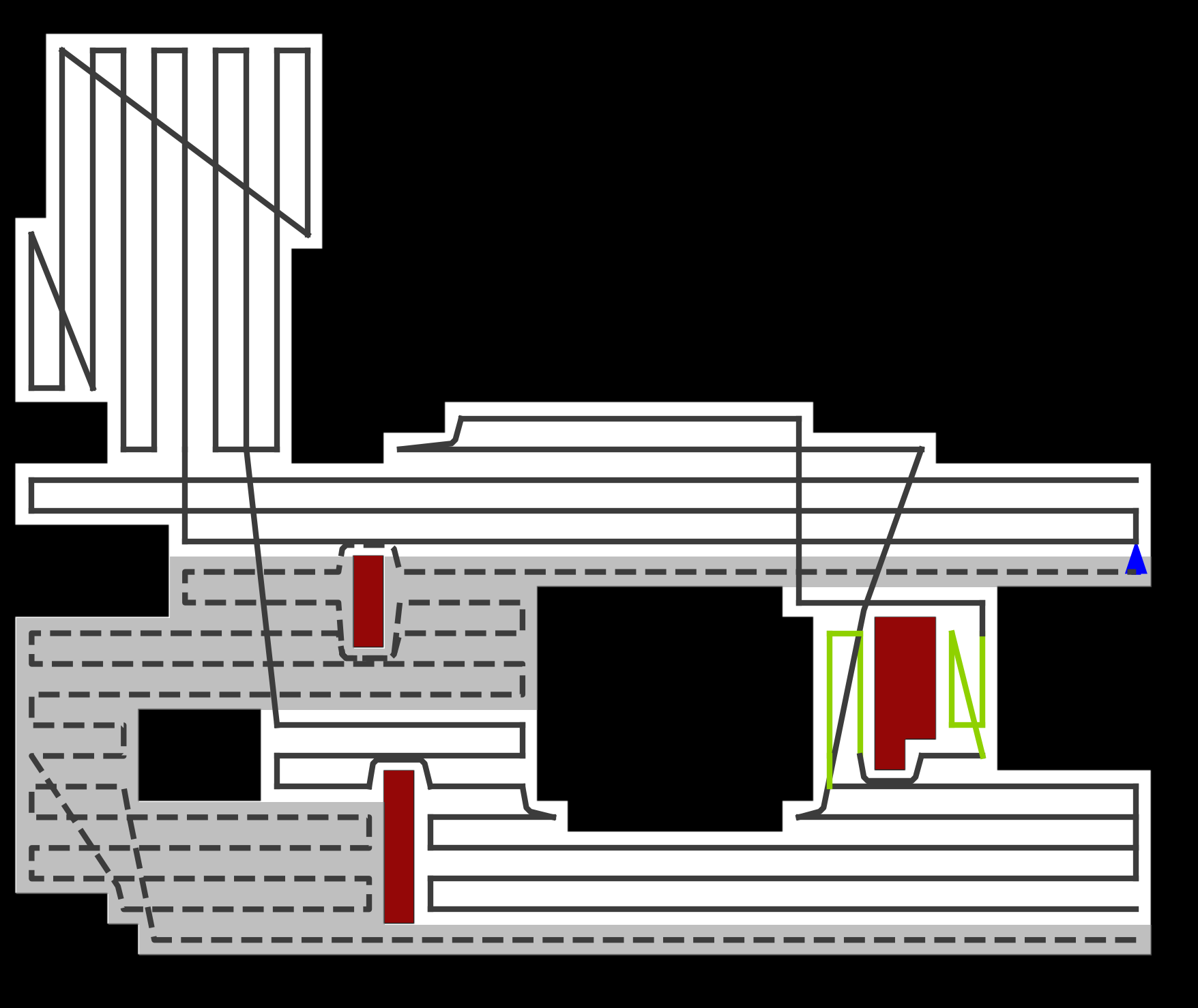}
    }
    \caption{Replanning coverage paths in partially unknown environments. (a) An initial coverage path for the known base environment, (b) the robot (blue triangle) covers part of the environment (gray region) and observes a new obstacle (red box) interrupting the initial path (red path), and (c) a near-optimal replanned path (green path) computed online using the proposed approach (OARP-Replan).}
    \label{fig:replanning-example}
\end{figure*}

One approach to coverage replanning is to stop executing the current coverage path and plan a new path for the uncovered environment with the new obstacle(s). However, coverage planning \docupdate{for arbitrary environments is generally NP-Hard} \cite{arkinApproximationAlgorithmsLawn2000} and this approach does not scale well for large environments with many unknown obstacles. Moreover, an obstacle may only affect a small section of the path, i.e., the unbroken parts of the path may still be drivable and near-optimal. An alternate approach is to only replan the affected section(s) of the path (\textit{local replanning}) while keeping the rest of the path intact. A simple example of local replanning would be where the robot simply drives around the obstacle whenever a collision is imminent. However, a more practical approach would be to use the time while driving towards the blockage to find more optimal solutions. In this paper, we propose an anytime coverage replanner that utilizes the time spent driving to conduct near-optimal replans on the provided coverage paths (see Fig. \ref{fig:replanning-example}).

In our previous work \cite{rameshOptimalPartitioningNonConvex2022a}, we developed the Optimal Axis-Parallel Rank Partitioning (OARP) approach to plan coverage paths that minimize the number of turns in the solution. Minimizing turns is motivated by path quality and optimality as robots can have decreased performance while turning (speed and coverage). OARP solves this using the following three-step framework: i) decomposing the environment into grid cells, ii) using the decomposition to compute the minimum number of axis-parallel (horizontal and vertical) \textit{ranks}, i.e. long rectangles that match the width of the robot's coverage tool while covering along a straight-line path, and iii) finding an optimal tour of the ranks to obtain the coverage path. In this paper, we extend OARP to provide a coverage replanning approach, namely \textit{OARP-Replan}, that replans coverage paths.


Our specific contributions are as follows:
\begin{enumerate}[1.]
\item We present an \docupdate{Integer Linear Program (ILP)} that replans the ranks of a coverage path such that a tour of the ranks (i.e. a coverage path) can be replanned within a time budget.
    \begin{enumerate}[a)]
    \item We leverage LP relaxations of this \docupdate{ILP} to replan ranks in polynomial time (efficient to compute).
    \item We present a second \docupdate{ILP} that reduces the number of decision variables, and therefore replans the coverage ranks faster than the first \docupdate{ILP}.    
\end{enumerate}
\item We provide an anytime coverage replanning method named \textit{OARP-Replan} that leverages the above \docupdate{ILPs} to conduct near-optimal replans to the coverage path within a given time budget.
\item \docupdate{Finally, we conduct simulations and robot experiments to validate our proposed approach. Our simulation results include comparisons against other state-of-the-art approaches, including offline approaches (i.e. all obstacles are known) \cite{bahnemannRevisitingBoustrophedonCoverage2021, agarwalAreaCoverageMultiple2022a}.}
\end{enumerate}
\subsection{Related Works}
Detailed surveys of coverage planning approaches can be found in \cite{galceranSurveyCoveragePath2013, bormannIndoorCoveragePath2018, tanComprehensiveReviewCoverage2021}. Broadly, they can be characterized into two classes: \textit{offline} and \textit{online} approaches. Offline approaches plan coverage paths with complete information about all obstacles in the environment, whereas online approaches start with little to no information about the obstacles in the environment and compute a path during the coverage operation. 

\textbf{\emph{Offline Coverage Planning:}} Offline approaches can be further categorized into 2 sub-classes: (i) \textit{exact decomposition} approaches, and (ii) \textit{grid-based} approaches. In exact decomposition methods, the environment is first decomposed into polygonal sub-regions or \textit{cells}. Each cell is then individually covered either by sweeping (i) back-and-forth along coverage parallel lines \cite{chosetCoveragePathPlanning1998}, (ii) along spiral paths \cite{cabreiraEnergyAwareSpiralCoverage2018}, or (iii) a combination of both \cite{brownConstrictionDecompositionMethod2016}. There are also turn-minimizing approaches that cover each cell in a locally optimal direction, generally determined using the cell's geometry \cite{bochkarevMinimizingTurnsRobot2016, bahnemannRevisitingBoustrophedonCoverage2021, dasMappingPlanningSample2014}. These cells are however complex to compute, and as a result, these methods do not scale well for online replanning. 

\docupdate{Grid-based approaches use a grid approximation of the given environment and compute a path covering all the grid cells in the approximation. This path can be computed either by solving the Traveling Salesman Problem (TSP) \cite{bormannIndoorCoveragePath2018}, using a spanning tree of the grid cells (Spanning Tree Coverage or STC) \cite{gabrielySpanningtreeBasedCoverage2001}, or using a wavefront algorithm to determine visitation order \cite{zelinskyUnifiedApproachPlanning1994}.} Turn-minimizing approaches of this class look to minimize the number of axis-parallel (horizontal or vertical) coverage lines or \textit{ranks} in the coverage path \cite{vandermeulenTurnminimizingMultirobotCoverage2019, luTMSTCPathPlanning2023}. While these methods were demonstrated to work well in offline settings, they lack online extensions that can replan a path within a given time budget. Our proposed method, OARP-Replan, extends the OARP approach from \cite{rameshOptimalPartitioningNonConvex2022a}, which is an offline grid-based turn-minimizing approach, to be applicable in online replanning scenarios.




\textbf{\emph{Online Coverage Planning:}} Online coverage planning approaches have gained recent interest, and as such, there is a wide range of different approaches in the literature. Initially, online exact decomposition approaches were investigated \cite{acarSensorbasedCoverageUnknown2002, vietBAOnlineComplete2013}, but these methods involve identifying changes to exact cells online, which does not scale well for large cluttered environments. Moreover, these methods do not aim to minimize turns in the coverage path.

Grid-based methods work better for such environments as one can easily update a grid with new occupancy information. Our proposed approach falls under this class. A popular approach of this class is Online STC \cite{gabrielyCompetitiveOnlineCoverage2003}, which computes a spanning tree of the grid cells online to update its coverage plan. \docupdate{However, given a square coverage tool of width $l$, STC is generally applied to coarse grids with cells of width $2l$ and does not generalize well to finer grids with cell width $l$.} Other grid-based online approaches use various local search and optimization methods to determine the coverage plan in an unknown environment. Song et al \cite{song2018} proposed an approach using artificial potential fields (APF) to determine a coverage path online. A similar technique to APFs is using bio-inspired neural networks to model and search the environment. This was originally proposed by Luo and Yang in \cite{luoBioinspiredNeuralNetwork2008} and has gained popularity in many recent works \cite{muthugalaEnergyefficientOnlineComplete2022, sunCompleteCoverageAutonomous2019a, godioBioinspiredNeuralNetworkBased2021}. However, these local search-based approaches are susceptible to deadlocking, where the robot does not know where to go next. Deadlocks are computationally expensive to get out of and affect overall coverage time. Moreover, these approaches generally minimize turns by penalizing them in a local search cost function, which does not guarantee minimum turns for the path overall. Our proposed approach avoids deadlocks (using an initial path) and guarantees minimum turns in the overall path when provided with sufficient planning time.

\docupdate{Online coverage planning shares similarities with exploration planning, which involves covering unknown environments using sensors \cite{dangGraphbasedSubterraneanExploration2020a}. Exploration planning algorithms use information theory \cite{corahEfficientOnlineMultirobot2017, huOffPolicyEvaluationOnline2023} and frontiers \cite{yamauchiFrontierbasedApproachAutonomous1997} to determine the next-best scan of the environment, rather than using decompositions. Recently, a decomposition-free online coverage planner was proposed that uses frontiers and search-based planning to identify the next-best region to cover \cite{kusnurCompleteDecompositionFreeCoverage2022}.} Online coverage replanning has also been investigated in 3D settings for applications in underwater inspection tasks \cite{galceranCoveragePathPlanning2015}, but they involve 3D maneuvers and are not necessarily feasible for robots operating in 2D planes.





\subsection{Organization}

The rest of this paper is organized as follows. In Section \ref{sec:prelims}, we provide some background on the coverage planning problem and the OARP approach from our previous work \cite{rameshOptimalPartitioningNonConvex2022a}. In Section \ref{sec:problem-def}, we define the problem of replanning a coverage path within a given time budget. In Section \ref{section:solution-approach}, we present the high-level framework of our proposed approach, OARP-Replan, that constrains the runtime to replan the coverage path, which solves the problem in two stages: (i) \textit{rank replanning}, and (ii) \textit{touring}. In Section \ref{sec:rank-replan}, we describe the rank replanning stage where we use \docupdate{ILPs} to replan the ranks of the coverage path such that the overall runtime is constrained. In Section \ref{sec:rank-touring}, we describe the touring stage where we formulate a Generalized Traveling Salesman Problem (GTSP) to compute a tour of the replanned ranks to obtain the replanned path. In Section \ref{sec:results}, we provide simulation results on a range of real-world environments and compare our proposed approach against state-of-the-art coverage planning methods in online and offline settings. \docupdate{In Section \ref{sec:robot-experiments}, we present experimental results of testing our replanning approach using an industrial-level autonomous cleaning robot.}
\section{Preliminaries}
\label{sec:prelims}

We first review the coverage planning problem and our previously proposed coverage planning approach, namely the optimal axis-parallel rank partitioning (OARP) approach~\cite{rameshOptimalPartitioningNonConvex2022a}.

\subsection{The Coverage Planning Problem}
Consider a 2D non-convex environment with known obstacles, where $\mathcal{W} \subseteq \mathbb{R}^2$ is a closed and bounded set that represents all free points within the environment. We now consider a robot with a state composed of its position \docupdate{$\boldsymbol{x} \in \mathcal{W}$, where $\boldsymbol{x}$ is a vector in $\mathbb{R}^2$}, and a heading $\theta \in [0,2\pi)$ radians. The robot is carrying a coverage tool with a footprint given by $\mathcal{A} \subseteq \mathbb{R}^2$. In this work, similar to \cite{arkinOptimalCoveringTours2005, vandermeulenTurnminimizingMultirobotCoverage2019, rameshOptimalPartitioningNonConvex2022a}, we assume that the coverage tool is a square of width $l > 0$. Let $\mathcal{A}(\boldsymbol{x}, \theta) \subseteq \mathcal{W}$ represent the placement of the tool with respect to the robot's state. \docupdate{In reality, a square tool may not be able to cover the entire environment, so we look to cover an approximation of the environment $\widetilde{\mathcal{W}} \subseteq \mathcal{W}$ that the tool can reach. Consider a set $\mathcal{P}$ of all possible paths in the environment, where each path $P \in \mathcal{P}$ is given by a sequence of poses, i.e. position and heading pairs $(\boldsymbol{x},\theta)$.} We now formally define the coverage planning problem (CPP) of interest as follows.



\begin{problem}[Coverage Planning Problem (CPP)]
\label{def:coverage}
\docupdate{Given an approximation $\widetilde{\mathcal{W}}$ of a 2D environment $\mathcal{W}$} and a robot carrying a coverage tool of footprint $\mathcal{A}$, plan a path $P$ from the set $\mathcal{P}$ of all possible paths which solves
\begin{align}
    &\min_{P \in \mathcal{P}} \quad J(P) \\
    &\textrm{s.t.} \quad \bigcup\limits_{(\boldsymbol{x},\theta) \in P} \mathcal{A}(\boldsymbol{x}, \theta) = \widetilde{\mathcal{W}} \textrm{,}
\end{align}
where $J(P)$ is the cost of the path $P$.
\end{problem}

In this paper, we consider the path cost $J(P)$ to be the time for the robot to cover the given environment along $P$.

\subsection{Optimal Axis-Parallel Rank Partitioning (OARP)}

We now briefly describe the coverage planning approach from our previous work \cite{rameshOptimalPartitioningNonConvex2022a}, namely the Optimal Axis-Parallel Rank Partitioning (OARP) approach. The focus of the OARP approach is to plan optimal coverage paths in which the coverage tool moves only in \textit{axis-parallel} directions (horizontal and vertical) during coverage. The OARP approach uses a three-step framework to plan coverage paths: (i) decompose the environment $\mathcal{W}$ into an Integral Orthogonal Polygon (IOP), which is a set of square \textit{grid cells} of size $l$, (ii) place the minimum number of axis-parallel (horizontal or vertical) \emph{ranks} (straight-line paths) to cover the IOP, and (iii) compute a \emph{tour} of the ranks to construct the coverage path. \docupdate{Here, the IOP forms the environment subset $\widetilde{\mathcal{W}}$ that we look to cover.} To place the minimum number of ranks in the IOP, OARP solves a linear program (LP) to compute the coverage orientation of each cell in the IOP. Once the ranks are obtained, a tour of the ranks is computed by solving a Generalized Travelling Salesman Problem (GTSP) that minimizes the cost of transitioning (driving) between ranks. The resulting coverage path is represented as a series of ranks connected by cost-minimizing \textit{transition paths}, which are collision-free dynamically feasible paths to traverse between ranks. Note that, unlike ranks, transition paths do not have to be axis-parallel.


\section{Problem Definition}
\label{sec:problem-def}

In this section, we define the problem of replanning the coverage path.

\subsection{Coverage Replanning Problem}

Let us consider the case where our robot is trying to follow its initial coverage path $P$ but it observes a set $\mathcal{O}$ of previously unknown obstacles. In other words, the actual environment is $\overline{\mathcal{W}} \subseteq \mathcal{W} \setminus \mathcal{O}$. We assume the robot is equipped with a sensor that can observe its environment within a sensor footprint relative to the robot's position, e.g., an optical range sensor. Suppose, at position $\boldsymbol{x}_i$, the robot detects a new obstacle interrupting the path. Let $\tau$ be the time in which the robot would reach the obstacle along its current path from $\boldsymbol{x}_i$. While the robot might be able to continue along $P$ for time $\tau$, the subsequent path must be replanned. One solution is to solve the full CPP for the new environment $\overline{\mathcal{W}}$. While this method would give us the optimal coverage path, it has a major drawback. If solving CPP takes longer than $\tau$, the robot may have to stop and wait for a new path, which is undesirable. To address this, we instead solve the coverage replanning problem that takes into account a time budget $\tau$.

Let $P_{\text{r}}$ be the remaining path after the first interruption by the obstacle. \docupdate{Let $\mathcal{W}_{\text{r}} \subseteq \widetilde{\mathcal{W}}$ be the area covered by $P_{\text{r}}$, where $\mathcal{O} \cap \mathcal{W}_{\text{r}} \neq \emptyset$}. The replanning problem is to plan a new coverage path $P_{\text{r}}'$ before $\tau$ such that it covers the same areas as $P_{\text{r}}$ without colliding into $\mathcal{O}$. We are now ready to pose the main problem we look to solve in this paper.

\begin{problem}[Coverage Replanning Problem (CRP)]
\label{problem:replanning}
\docupdate{Given a path $P$ covering $\widetilde{\mathcal{W}}$}, a set of newly observed obstacles $\mathcal{O}$, a time $\tau$ until the first such obstacle is encountered along $P$, the remaining path $P_{\text{r}}$ after the first encounter, and the area \docupdate{$\mathcal{W}_{\text{r}} \subseteq \widetilde{\mathcal{W}}$} covered by $P_{\text{r}}$, find a new coverage path $P_{\text{r}}'$ within time $\tau$ which solves
\begin{align}
    &\min_{P_{\text{r}}' \in \mathcal{P}} \quad J(P_{\text{r}}') \\
    &\textrm{s.t.} \quad \bigcup\limits_{(\boldsymbol{x}, \theta) \in P_{\text{r}}'} \mathcal{A}(\boldsymbol{x}, \theta) = \mathcal{W}_{\textrm{r}} \setminus \mathcal{O} \textrm{.}
\end{align}
\end{problem}
CRP is similar to CPP when $\tau$ is infinite and $P = P_{\text{r}}$. However, for a finite $\tau$, CRP must find a trade-off between minimizing coverage cost and ensuring a path is computed within $\tau$. The larger the $\tau$ value, the closer the coverage cost is to the optimal path (CPP solution).

\subsection{Rank-based Coverage Replanning}

CPP is NP-Hard \cite{arkinApproximationAlgorithmsLawn2000}, which indicates that Problem \ref{problem:replanning} is intractable. Therefore, we look for relaxations of Problem \ref{problem:replanning} that are more solvable. A common approach is to solve this problem in stages. While this may remove some optimal solutions in the search space, it reduces the complexity of the problem and makes it tractable.

In this paper, we focus on replanning initial coverage paths that are represented as a series of axis-parallel (horizontal/vertical) ranks covering an IOP decomposition of our initial environment, e.g., paths generated by OARP. With this setup, we approach Problem \ref{problem:replanning} in two stages: (i) rank replanning, and (ii) touring. In the rank replanning stage, we compute the minimum number of axis-parallel ranks to cover the IOP of the remaining region $\mathcal{W}_{\text{r}} \setminus \mathcal{O}$. In the touring stage, we obtain the replanned path $P_{\text{r}}'$ by computing a tour of the ranks that minimizes the cost of transitioning between the ranks. We also look to ensure that both stages are completed within the time budget $\tau$. The following problem captures this requirement.

\begin{problem}[Rank-based Coverage Replanning]
\label{problem:rank-based-replanning}
Given a set of newly observed obstacles $\mathcal{O}$ that interrupts the remainder $P_{\text{r}}$ of the coverage path $P$, and a time budget $\tau$ before the robot's first encounter with an obstacle in $\mathcal{O}$, compute the following within time $\tau$: (i) a minimum set of ranks that covers the IOP representation of the uncovered free region $\mathcal{W}_{\text{r}} \setminus \mathcal{O}$ and (ii) a tour of the ranks that minimize the total cost of transition between ranks.
\end{problem}


\section{Solution Framework}
\label{section:solution-approach}

\begin{figure}
    \centering
    \includegraphics[width=0.9\linewidth]{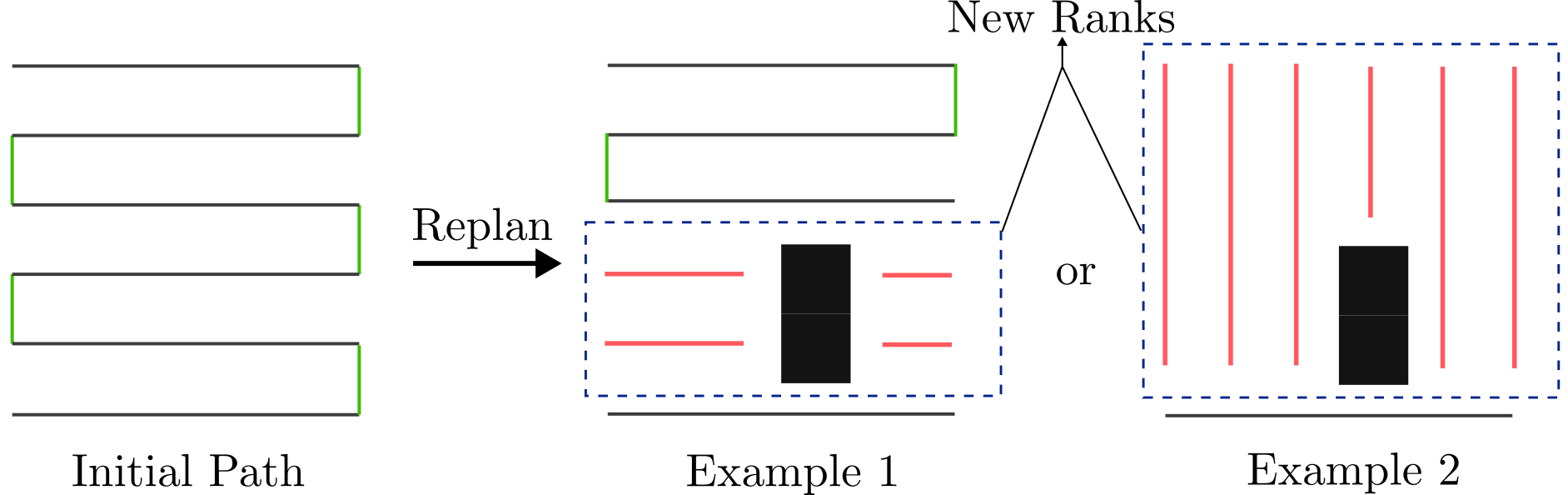}
    \caption{The effect of an obstacle (black box) on replanning a coverage path consisting of ranks (black lines) and transition paths (green lines). The ranks of the initial path (left) when interrupted can be replanned in many ways, two of which are shown here. Example 1 has fewer new ranks (red lines in dotted box) than example 2 and is preferable for shorter planning time. Example 2 has fewer total ranks and is preferable for shorter drive time.
    \vspace{-2mm}}
    \label{fig:eg-rank-change}
\end{figure}

In this section, we describe our solution framework for coverage replanning. Unlike OARP, we look to re-use sections of the initial coverage plan which adds an additional coupling between the rank replanning and touring stages. In this section, we first explore this coupling with an example, and then present the high-level algorithm for the proposed approach.

\subsection{The effect of changing ranks}

We aim to solve rank replanning efficiently so that the majority of the time budget is available for the touring stage, which is formulated as a Generalized Travelling Salesman Problem (GTSP), an NP-hard problem. However, the rank replanning output affects the GTSP input size, and in turn, the runtime. Therefore, we design the rank replanning stage such that the tour of the replanned ranks can be computed within the time budget $\tau$.

To better understand this coupling, consider the example coverage path shown in Fig. \ref{fig:eg-rank-change} with an initial set of horizontal ranks. Let us now introduce an obstacle, for which rank replanning returns two different sets of ranks (example 1 and example 2). Each set contains ranks that are unchanged from the initial path and \textit{new ranks}, i.e., ranks that were not in the initial solution. To reduce the computational complexity of the touring step, we reuse sections of the initial path to cover unchanged ranks. This allows us to spend the time budget on connecting new ranks to the unchanged path sections. As a result, the GTSP problem size is driven by the number of \textit{new ranks}. In Fig. \ref{fig:eg-rank-change}, Example 1's second stage has a GTSP size of 6 sets (4 new ranks + 2 retained paths), whereas Example 2 has a GTSP size of 7 sets (6 new ranks + 1 retained path of one rank) despite having fewer total ranks. Note that the number of unchanged disconnected path sections is also bounded by the number of new ranks.

Let $m$ be the number of total ranks in the replanned solution and $m_{\text{new}} \leq m$ be the number of new ranks. The GTSP input size is dependent on $m_{\text{new}}$ and the estimated GTSP runtime would be a monotonically increasing function $\hat{T}(m_{\text{new}})$. Therefore, given a time budget $\tau$, we bound $m_{\text{new}}$ such that $\hat{T}(m_{\text{new}}) \leq \tau$. In other words, a larger $\tau$ allows us to change more ranks ($m_{\text{new}}$) so that we can minimize the total number of ranks ($m$). To compute this bound, we first determine $\Hat{T}$ using data collected from previous GTSP runs (e.g. data from the offline planning of the initial coverage path). We then set $\Bar{m} = \Hat{T}^{-1}(\tau)$ as the maximum GTSP instance size, where $\Hat{T}^{-1}$ is the inverse of $\Hat{T}$. Thus, we add the following constraint in rank replanning to ensure that the touring stage finishes within time $\tau$:
\begin{align}
    m_{\text{new}} \leq \Hat{T}^{-1}(\tau) = \Bar{m}. \label{eq:bound_on_del_m}
\end{align}

In Section \ref{sec:rank-replan}, we will propose a rank replanning approach that constrains the number of new ranks $m_{\text{new}}$ in its output within a budget $\Bar{m}$ derived from $\tau$. In Section \ref{sec:results-setup}, we provide more details about how $\Hat{T}$ can be determined. 

\subsection{Algorithm}
\label{sec:high-level-approach}

\begin{figure}
    \centering
    \includegraphics[width=0.9\linewidth]{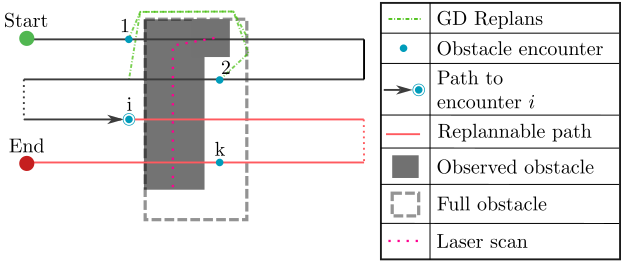}
    \caption{
    A run of OARP-Replan with one obstacle (dotted grey box). \docupdate{The visible part of the obstacle (pink dots) is used to identify blocked grid cells (grey) and obstacle encounters (blue dots) along the current path.} For each encounter $i \leq k$, the planner attempts to replan the remaining path (red path) within the time $\tau_i$ to reach the encounter (black path). To compute $\tau_i$, previous encounters are replanned using GD replan (green dotted lines).
    \vspace{-2mm}}
    \label{fig:example-run}
\end{figure}

We now describe the high-level algorithm of our proposed approach, namely OARP-Replan, for which the pseudocode is given in Algorithm \ref{alg:oarp-replan}. \docupdate{We consider a robot that is equipped with a sensor similar to that of a laser scanner (LIDAR) that is capable of detecting obstacles (or parts of obstacles) in the environment. OARP-Replan is run when potential collisions or \textit{encounters} are detected along the current path $P$ using the robot's sensor model. More specifically, we determine the grid cells in the IOP decomposition that cannot be covered anymore due to newly detected obstacles (if the obstacle was partially observed, this process would be repeated for each scan that reveals more of the obstacle).} Let $k$ be the number of \textit{all} projected \textit{encounters} along $P$ (each obstacle may cause multiple encounters). For each encounter, the robot has the option of replanning the remaining path (the preferred solution) or simply going around the obstacle while covering the interrupted rank (if there is no time to replan). We refer to the latter method as the \textit{greedy detour} (GD) replan, which can be performed using a local planner with little to no computational time and thus no robot stoppage. We propose only using GD replan as a fallback method in the case where the main planner does not have enough time to replan the path.

Fig. \ref{fig:example-run} illustrates an example run of the algorithm. \docupdate{Note that we replan using the obstacle encounters that have been observed so far by the robot. The encounters that were not observed (even from an obstacle that was partially observed) is replanned once the robot has observed them later in the coverage run.} We iterate through each obstacle encounter $i \leq k$ and evaluate whether the main planner can replan the path after the $i^\text{th}$ encounter. The path for all previous obstacle encounters ($\{1, 2, \dots i-1\}$) is replanned using GD replan. In other words, we look for an obstacle encounter that gives the planner enough time to replan the path after the encounter. Let $\tau_i$ be the time for the robot to reach the $i^{\text{th}}$ encounter along the current path, which will be the time budget for the replan. 

To evaluate if the path after encounter $i$ can be replanned, the planner runs a sub-routine named RANK-REPLAN to get a set of replanned ranks for which a tour can be computed in time $\tau_i$ (Line \ref{line:replan}). If RANK-REPLAN returns a set of ranks, we compute a tour of the ranks using another sub-routine called TOUR-REPLAN (Line \ref{line:tour}) that solves the GTSP to obtain our replanned path. This sub-routine runs as a background process while the robot continues to follow the current coverage path. However, if RANK-REPLAN returns an empty set, then this indicates that there is no solution that will finish within the expected time budget. In that case, GD replan is used to replan encounter $i$ (Line \ref{line:backup}), and the procedure is repeated for the next encounter. In the end, we obtain a replanned path where some (or all) encounters are replanned using GD replan, while the remaining path undergoes a near-optimal replan before the robot reaches the corresponding encounter.

\begin{algorithm}
\caption{OARP-Replan}
\label{alg:oarp-replan}
 \begin{algorithmic}[1]
 \renewcommand{\algorithmicrequire}{\textbf{Input:}}
 \renewcommand{\algorithmicensure}{\textbf{Output:}}
 \Require Current path $P$, \docupdate{Environment IOP $\widetilde{\mathcal{W}}$}
 \Ensure Replanned path $P'$
 \For{\docupdate{each encounter $i \leq k$ on $P$}}
 \vspace{0.5mm}
 \State $\mathcal{W}_i \gets$ Part of $\widetilde{\mathcal{W}}$ to cover after $i^\text{th}$ encounter
 \State $R \gets$ Ranks from coverage path $P$
 \State $\tau_i \gets$ Time to reach $i^\text{th}$ encounter
 \State $\Bar{m}_i \gets \hat{T}^{-1}(\tau_i)$ 
 \Comment{Maximum instance size given $\tau_i$}
\State $R' \gets $ RANK-REPLAN($\mathcal{W}_i$, $R$, $\Bar{m}_i$) \label{line:replan}
\If{$R' \neq \emptyset$}
    \State $P' \gets $ TOUR-REPLAN($R'$, $R$) \label{line:tour}
    \State \Return $P'$
\Else
\State Apply GD replanning to $i^{th}$ encounter \label{line:backup}
\EndIf
\EndFor
\end{algorithmic}
\end{algorithm}

In our implementation, we run OARP-Replan using the path that the robot is \textit{currently executing}, i.e. if a new obstacle is detected before $P'$ is returned, we start a new replan of $P$. In the rest of the paper, we will expand on how the RANK-REPLAN and TOUR-REPLAN sub-routines work. 
\section{The RANK-REPLAN Subroutine}
\label{sec:rank-replan}

In this section, we extend the \docupdate{integer linear program (ILP)} from the OARP approach in our previous work \cite{rameshOptimalPartitioningNonConvex2022a} to formulate the RANK-REPLAN subroutine that minimizes the number of replanned ranks while constraining the number of new ranks within a budget $\Bar{m}$. We first provide a brief review of the original \docupdate{ILP} with some additional variables to allow for extension. Then, we propose two extensions to this \docupdate{ILP}: the first extension solves the rank replanning step optimally, and the second extension solves it efficiently with fewer \docupdate{ILP} variables (no optimality guarantees). The user may choose either \docupdate{ILP} based on their replanning preferences (optimality vs efficiency). 

\subsection{The base \docupdate{ILP}}
\label{sec:base-milp}

Consider an integral orthogonal polygon (IOP) containing $n$ grid cells representing the input environment, where each cell $c_i$ for $1 \leq i \leq n$ is of size $l \times l$ ($l$ is the tool width). The goal of the \docupdate{ILP} is to assign a \textit{coverage orientation} (horizontal or vertical) for each cell $c_i$ such that the number of ranks covering the cells is minimized. We obtain ranks by merging \textit{similarly-oriented neighboring} grid cells. To this end, we introduce variables to represent cell orientations and count the resulting ranks after all merges are complete.

\begin{figure}
    \centering
    \includegraphics[width=\linewidth]{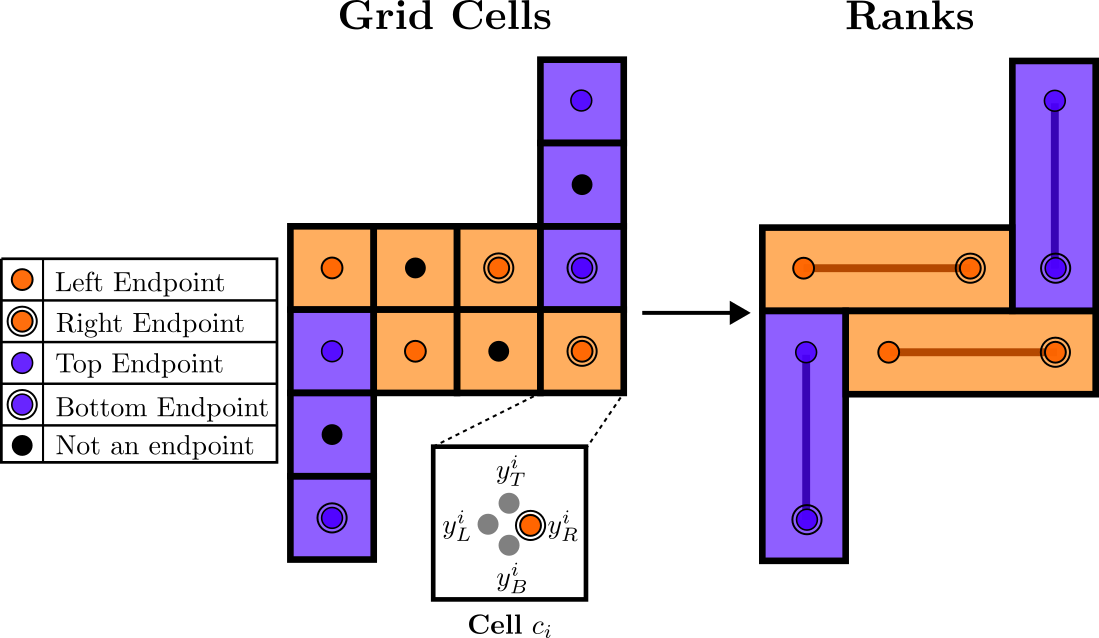}
    \caption{Integral orthogonal polygon (IOP) of an environment where each grid cell is either oriented horizontally (orange) or vertically (purple). The left, right, top, and bottom rank endpoints are also shown (see legend). Each cell $c_i$ (zoomed-in cell) has four endpoint indicator variables. If the cell is an endpoint (e.g. $c_i$ is a right endpoint), the corresponding variable is active ($y_R^i = 1$).
    \vspace{-2mm}}
    \label{fig:all-endpoints}
\end{figure}

Let $x_H^i$ and $x_V^i$ be binary variables that denote whether cell $c_i$ is oriented \textit{horizontally} or \textit{vertically} (cannot be both). Fig. \ref{fig:all-endpoints} shows an example IOP with oriented grid cells and their corresponding ranks. We also add variables to identify the rank \textit{endpoints}. Let $y_L^i$ and $y_R^i$ be binary variables that indicate whether the cell, $c_i$, is the \textit{left} endpoint and/or the \textit{right} endpoint of a horizontal rank. Similarly, let $y_T^i$ and $y_B^i$ be binary variables that indicate whether $c_i$ is a \textit{top} and/or \textit{bottom} endpoint of a vertical rank. This notation is different from the original OARP approach, which only has indicator variables for the left and top endpoints. Fig. \ref{fig:all-endpoints} also illustrates the endpoint variables and rank endpoints for the oriented grid cells. 

Let $\boldsymbol{x_H}, \boldsymbol{x_V}, \boldsymbol{y_L}, \boldsymbol{y_R}, \boldsymbol{y_T}, \boldsymbol{y_B}$ be $n$ dimensional binary vectors of the corresponding variables for all grid cells. Let $\boldsymbol{y_H} = \begin{bmatrix}
\boldsymbol{y_L}^T & \boldsymbol{y_R}^T
\end{bmatrix}^T$ and $\boldsymbol{y_V} = \begin{bmatrix}
\boldsymbol{y_T}^T & \boldsymbol{y_B}^T
\end{bmatrix}^T$ be the vectors of all horizontal rank and vertical rank endpoints respectively, where $\boldsymbol{y_H}, \boldsymbol{y_V} \in \{0,1\}^{2n}$. Since each rank has exactly two endpoints, the number of coverage ranks $m$ is given by half the sum of all rank endpoints:
\begin{align}
    m = \frac{1}{2} (\boldsymbol{1}^T \boldsymbol{y_H} + \boldsymbol{1}^T \boldsymbol{y_V}),
    \label{eq:m_original}
\end{align} 
where $\boldsymbol{1}$ is the column vector of ones. This constitutes the objective of the optimization problem.

To formulate the \docupdate{ILP} constraints, we must encode the relationship between the orientations ($x$ variables) and endpoints ($y$ variables). Following from our previous work \cite{rameshOptimalPartitioningNonConvex2022a}, we use directed graphs in the IOP to detect whether a neighboring cell can be merged. Let $G_L$ and $G_R$ be directed graphs composed of all horizontal path flows for each grid cell from its left and right neighbors respectively. For cells at the IOP border, we add artificial nodes called \textit{border identifiers} to simulate path flow into those cells. Similarly, let $G_T$ and $G_B$ be the graphs composed of all vertical path flows for each grid cell from its top and bottom neighbors respectively. \docupdate{Fig. \ref{fig:directed_graphs} illustrates these graphs for an example IOP.} From these graphs, let $A_L, A_R, A_T, $ and $A_B$ be matrices representing the corresponding graphs, where each \textit{row} of a matrix signifies a directed edge in the graph, with a $-1$ for a source grid cell (outgoing edge), a $+1$ for a sink grid cell (incoming edge), and $0$s otherwise. These matrices resemble the node-arc incidence (NAI) matrices for the graphs $G_L, G_R, G_T, $ and $G_B$.

We now pose the rank minimization \docupdate{ILP}, namely \docupdate{ILP}-0, that we will extend in the rest of this section:
\begin{align}
\nonumber\text{\textbf{\docupdate{ILP}-0:}} \hspace{5mm} & \\
\min_{\boldsymbol{x_H}, \boldsymbol{x_V}, \boldsymbol{y_H}, \boldsymbol{y_V}} \quad & m \label{eq:obj-0}\\
\textrm{s.t.} \quad & \begin{bmatrix}
A_L^T & A_R^T
\end{bmatrix}^T\boldsymbol{x_H} - 
\boldsymbol{y_H}
\leq \boldsymbol{0} \label{eq:c1-new}\\
 & \begin{bmatrix}
A_T^T & A_B^T
\end{bmatrix}^T\boldsymbol{x_V} -
\boldsymbol{y_V}
\leq \boldsymbol{0} \label{eq:c2-new} \\
& \boldsymbol{x_H} + \boldsymbol{x_V} = \boldsymbol{1} \label{eq:c3}\\
& \boldsymbol{x_H}, \boldsymbol{x_V} \in \{0, 1\}^{n} \\
& \boldsymbol{y_H}, \boldsymbol{y_V} \in \{0, 1\}^{2n}.
\end{align}
In this \docupdate{ILP}, equations (\ref{eq:c1-new}) and (\ref{eq:c2-new}) ensure that, for an \textit{optimal} solution of \docupdate{ILP}-0, vectors $\boldsymbol{y_H}$ and $\boldsymbol{y_V}$ correspond to the endpoints of the optimal number of ranks covering the IOP. \docupdate{This follows from the minimization of $m$, which in-turn minimizes $\boldsymbol{y_H}$ and $\boldsymbol{y_V}$ (Eq. (\ref{eq:m_original})), thereby tightening equations (\ref{eq:c1-new}) and (\ref{eq:c2-new}).} Equation (\ref{eq:c3}) ensures that each cell is only assigned one orientation (can be either horizontal or vertical). For an in-depth discussion about the formulation of these constraints, we refer the reader to \cite{rameshOptimalPartitioningNonConvex2022a}. 

\begin{figure}
    \centering
    \includegraphics[width=\linewidth]{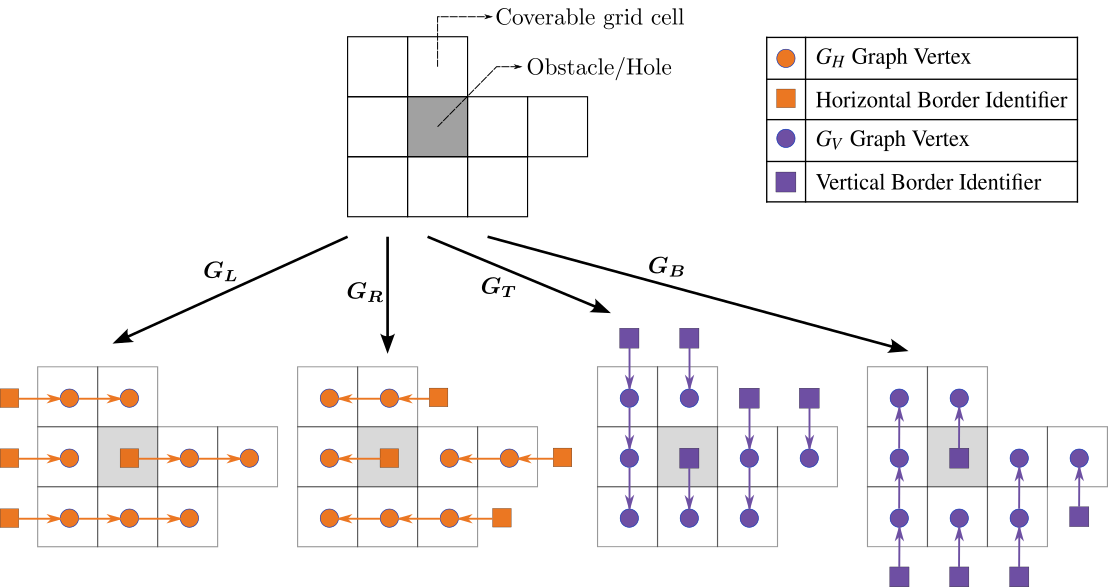}
    \caption{\docupdate{The directed graphs $G_L, G_R, G_T, $ and $G_B$ for the example IOP.}
    \vspace{-2mm}}
    \label{fig:directed_graphs}
\end{figure}

\subsection{Constraining new ranks}
\label{sec:rank-constraint}

We now look to constrain the number of new ranks $m_\text{new}$ to be within the rank budget $\Bar{m}$. To compute $m_\text{new}$, we identify the new ranks in the replanned solution by comparing them with the ranks from the initial path. This comparison requires an encoding of ranks using the \docupdate{ILP} variables. To this end, we propose an encoding based on perfect matching \cite{korteCombinatorialOptimizationTheory2006} that matches the endpoints for each rank in the \docupdate{ILP} solution. The new ranks are therefore the difference between the old and new endpoint matchings.

To simplify the explanation of the encoding, we first look at the horizontal case and simply extend the same ideas to the vertical case. Consider a horizontal row of \textit{connected} (unobstructed) grid cells in the IOP as depicted in Fig. \ref{fig:exact-measure-graph}. Now, consider the case where the row contains a horizontal rank with a left endpoint at a cell $c_i$. This indicates that the rank either (i) covers cells to the right of $c_i$ or, (ii) covers only $c_i$. In both cases, we observe that there must exist a cell to the right of $c_i$ (including $c_i$) that is the right endpoint of this rank. We look to identify this rank by matching its left and right endpoints.

We formulate this with variables that represent directed edges added between $c_i$ and every cell $c_r$ to its right, as any of them could be a right endpoint. We also add an edge from $c_i$ to itself to account for a single-cell rank. For each edge between $c_i$ and a cell $c_r$, we add a binary variable $z_{ir} \in \{0,1\}$ that takes a value of $1$ if the edge is included in the matching (i.e. $c_r$ is a right endpoint). To ensure that $c_i$ and $c_r$ are matched only if they are left and right endpoints respectively, we enforce the following constraints: 
\begin{align}
    &\sum_{\forall r}z_{ir} = y_L^i, \\
    &\sum_{\forall i}z_{ir} = y_R^r,\\
    &z_{ir} \in \{0,1\} \quad \forall 1 \leq i, j \leq n.
\end{align}

\begin{figure}
    \centering
    \includegraphics[width=0.9\linewidth]{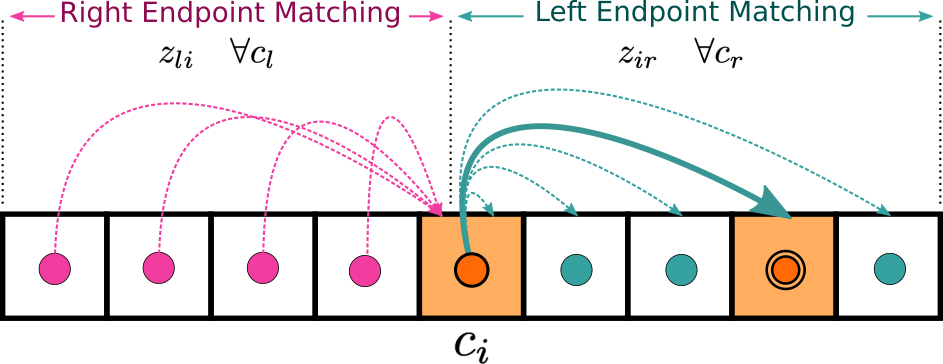}
    \caption{The matching edges used to detect ranks for a given row of grid cells. Each cell $c_i$ has edges that match it with another cell in the row if $c_i$ is a right (\textit{using magenta edges}) or left (\textit{using cyan edges}) endpoint. Here, we illustrate an example where $c_i$ is a left endpoint and is matched with a right endpoint in the row (bold cyan edge). No matching occurs if $c_i$ is not an endpoint.}
    \label{fig:exact-measure-graph}
\end{figure}

\begin{lemma}
    Consider a horizontal row of connected IOP grid cells and two cells $c_i$ and $c_r$ in this row where $c_r$ is to the right of $c_i$. Then there exists a horizontal rank with its left endpoint at $c_i$ and right endpoint at $c_r$ iff $z_{ir} = 1$.
    \label{lem:rank-z-correlation}
\end{lemma}
\begin{proof}
    ($\Rightarrow$ direction) If there exists a rank between $c_i$ and $c_r$, then $y_L^i = 1$ and $y_R^r = 1$. As a result, we have the following for $c_i$:
    \begin{align*}
    \sum_{\forall r}z_{ir} = 1, \quad \text{and} \quad \sum_{\forall i}z_{ir} = 1.
    \end{align*}
    Assume that $z_{ir} = 0$ (see Fig. \ref{fig:proof-setup}(a)). This indicates that $c_i$ must be matched with another right endpoint $c_k$ to its right, i.e. $z_{ik} = 1$. Since there exists a rank between $c_i$ and $c_r$, $c_k$ cannot be between the two cells. Therefore, $c_k$ must be to the right of $c_r$. WLOG, let us assume that the rank between $c_i$ and $c_r$ is the left-most horizontal rank in the row, meaning that there are no other left endpoints for $c_r$ to match with. This is an infeasible solution for the set of constraints unless $z_{ir} = 1$, thereby proving the statement.
    

    ($\Leftarrow$ direction) If $z_{ir} = 1$, then $y_L^i$ and $y_R^r$ have to both be $1$, indicating that they are left and right endpoints respectively. However, consider the case where there is a vertically-oriented cell between them, indicating that there is no horizontal rank between $c_i$ and $c_r$ (see Fig. \ref{fig:proof-setup}(b)). In this case, there must exist another cell $c_k$ that is a right endpoint (to the left of the vertical cell), where $y_R^k = 1$. WLOG, let us assume that $c_i$ is the only endpoint to the left of $c_k$. This indicates that $c_k$ has no other left endpoint to match with, which in turn gives an infeasible solution unless $z_{ik} = 1$ and $z_{ir} = 0$. This contradicts the fact that $z_{ir} = 1$, and therefore shows that there must exist a rank between $c_i$ and $c_r$ if $z_{ir} = 1$. This proves the lemma.
\end{proof}

\begin{figure}
\centering
\subcaptionbox{$\Rightarrow$ direction}
{\centering
\includegraphics[width=0.8\linewidth]{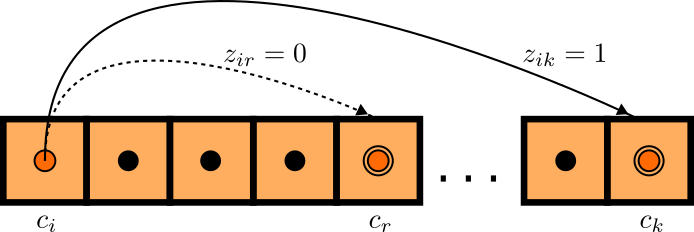}
}
\par\bigskip
\subcaptionbox{$\Leftarrow$ direction}
{\centering
\includegraphics[width=0.8\linewidth]{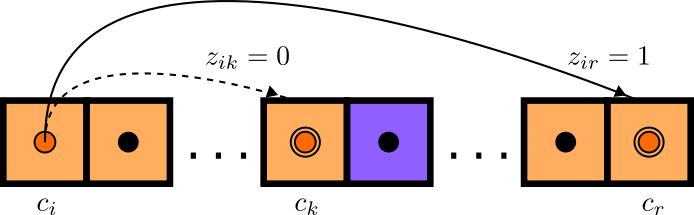}
}
\caption{Illustrations of the cases considered for Lemma \ref{lem:rank-z-correlation}'s proof.
\vspace{-2mm}}
\label{fig:proof-setup}
\end{figure}

Following Lemma \ref{lem:rank-z-correlation}, we now have a way of representing a horizontal rank as a matching of its endpoints. We extract each horizontal row of connected grid cells in the given IOP and add the following constraints for each cell $c_i$ in the row:
\begin{align}
    &\sum_{\forall r}z_{ir} = y_L^i, \quad \text{where $c_r$ is to the right of $c_i$}, \label{eq:z_left} \\
    &\sum_{\forall l}z_{li} = y_R^i, \quad \text{where $c_l$ is to the left of $c_i$}, \label{eq:z_right} \\
    &z_{ij} \in \{0,1\} \quad \forall 1 \leq i, j \leq n. \label{eq:z_h_bounds}
\end{align}

We simply extend the setup to represent vertical ranks as a matching of its endpoints in a vertical row of connected grid cells. For each cell $c_i$ in the vertical row, we add the following constraints:
\begin{align}
    &\sum_{\forall b}z_{ib} = y_T^i, \quad \text{where $c_b$ is to the bottom of $c_i$}, \label{eq:z_top} \\
    &\sum_{\forall t}z_{ti} = y_B^i, \quad \text{where $c_t$ is to the top of $c_i$}\label{eq:z_bottom} \\
    &z_{ij} \in \{0,1\} \quad \forall 1 \leq i, j \leq n. \label{eq:z_v_bounds}
\end{align}

Now, let $\Bar{z}_{ij}$ be an input to the \docupdate{ILP} such that $\Bar{z}_{ij}$ is $1$ if the initial path contains a rank with cells $c_i$ and $c_j$ as endpoints and $0$ otherwise. From Lemma \ref{lem:rank-z-correlation}, if the \docupdate{ILP} solution has a rank with $c_i$ and $c_j$ as endpoints ($z_{ij} = 1$), then it is a new rank if $\Bar{z}_{ij} = 0$. Using this relation for all pairs of cells $c_i$ and $c_j$ with matching edges, we compute $m_\text{new}$ as follows:
\begin{align}
    m_\text{new} = \sum_{\forall i}\sum_{\forall j} (1 - \Bar{z}_{ij}) z_{ij}. \label{eq:m_new_exact}
\end{align}

We now propose our extended \docupdate{ILP}, namely \docupdate{ILP}-1, that minimizes the number of ranks in the replanned solution while constraining the number of new ranks to be within a budget $\Bar{m}$. We also add the term $\epsilon \text{ } m_\text{new}$ to the \docupdate{ILP}'s objective function, where $\epsilon > 0$ is used to compute the rank-minimizing solution with the smallest $m_\text{new}$. \docupdate{Specifically, in the case of a large budget $\Bar{m}$, the additional term helps prevent rank changes that do not improve the number of ranks in the solution.} The resulting \docupdate{ILP} is as follows:
\begin{align}
\nonumber \text{\textbf{\docupdate{ILP}-1:}} \hspace{5mm} & \\
\min_{\boldsymbol{x_H}, \boldsymbol{x_V}, \boldsymbol{y_H}, \boldsymbol{y_V}} \quad & m + \epsilon \text{ } m_\text{new} \label{eq:obj-1}\\
\textrm{s.t.} \quad
& \nonumber \text{Eqs. (\ref{eq:c1-new}), (\ref{eq:c2-new}), (\ref{eq:c3}),} \\
& \nonumber \text{Eqs. (\ref{eq:z_left}), (\ref{eq:z_right}), (\ref{eq:z_top}) (\ref{eq:z_bottom}),} \quad \forall 1 \leq i \leq n,\\
& m_\text{new} \leq \Bar{m}, \label{eq:m_bound}\\
& z_{ij} \in \{0,1\} \quad \forall 1 \leq i, j \leq n, \label{eq:z_bounds} \\
& \boldsymbol{x_H}, \boldsymbol{x_V} \in \{0, 1\}^{n}, \\
& \boldsymbol{y_H}, \boldsymbol{y_V} \in \{0, 1\}^{2n}.
\end{align}
We now show that we can configure \docupdate{ILP}-1 to solve the rank replanning stage (Problem \ref{problem:rank-based-replanning}(i)) \textit{optimally}, i.e. the solution (i) minimizes the number of ranks $m$, (ii) subject to the constraint that $m_\text{new} \leq \Bar{m}$.
\begin{proposition}
    If $\epsilon < 1/\Bar{m}$, then \docupdate{ILP}-1 computes optimal solutions to Problem \ref{problem:rank-based-replanning}(i).    
    \label{prop:optimality}
\end{proposition}
\begin{proof}
    Based on our criteria of optimality, Problem \ref{problem:rank-based-replanning}(i) can be simplified as follows:
    \begin{align*}
        \nonumber \text{\textbf{P1:}} \quad \min \quad & m \\
    \textrm{s.t.} \quad & m_\text{new} \leq \Bar{m}.
    \end{align*}
    
    Any solution to \docupdate{ILP}-1 must satisfy Eq. \ref{eq:m_bound} and is therefore a feasible solution to P1. What remains to be shown is whether \docupdate{ILP}-1 computes the solution with minimum $m$, i.e., there exists an $\epsilon > 0$ such \docupdate{ILP}-1 returns an optimizer to P1. To show this, consider the simplified version of \docupdate{ILP}-1:
    \begin{align*}
        \nonumber \text{\textbf{P2:}} \quad \min \quad & m + \epsilon \text{ } m_\text{new}\\
    \textrm{s.t.} \quad & m_\text{new} \leq \Bar{m}.
    \end{align*}
    Note that since the constraints of P1 and P2 are the same, they share the same feasible solution space. Let $m^*$ and $m^{**}$ be the number of ranks corresponding to the optimizers to P1 and P2 respectively, where $m_\text{new}^*$ and $m_\text{new}^{**}$ are the corresponding values of $m_\text{new}$. Clearly, $m^* \leq m^{**}$.

    Let us assume, by way of contradiction, that $m^* \neq m^{**}$. Since $m$ is an integer, we have $m^{**} - m^* \geq 1$. Since the optimizer of P1 is feasible for P2, we have the following:
    \begin{align*}
        m^* + \epsilon \text{ } m_\text{new}^* &\geq m^\text{**} + \epsilon \text{ } m_\text{neq}^{**} \\
        \epsilon \text{ } (m_\text{new}^* - m_\text{new}^\text{**}) &\geq m^\text{**} - m^\text{*} \geq 1.
    \end{align*}
    From the constraints of P1 and P2, we have the following:
    \begin{align*}
        (m_\text{new}^* - m_\text{new}^\text{**}) &\leq m_\text{new}^* \leq \Bar{m}.
    \end{align*}
    Combining the inequalities, we get:
    \begin{align*}
        \epsilon \text{ } \Bar{m} &\geq 1 \\
        \epsilon &\geq 1/\Bar{m}.
    \end{align*}
    Following this, for values of $\epsilon < 1/\Bar{m}$, our assumption that $m^* \neq m^{**}$ is contradicted, which proves our proposition. We choose $\epsilon = 1/(\Bar{m}+1)$ for the rest of this work.
\end{proof}
\begin{remark*}[Linear Relaxation]
\label{remark:linear-relaxation}
    The linear relaxation of \docupdate{ILP}-1 is obtained by replacing the integrality constraints for the decision variables \docupdate{with real-valued constraints as follows:
    \begin{align*}
        &0 \leq z_{ij} \leq 1 \quad \forall 1 \leq i, j \leq n,\\
        &\boldsymbol{0} \leq \boldsymbol{y_H}, \boldsymbol{y_V} \leq \boldsymbol{1}, \\
        &\boldsymbol{x_H}, \boldsymbol{x_V} \geq \boldsymbol{0}.
    \end{align*}}
    We observed that solving this relaxation consistently gives integral solutions despite the constraints not being totally unimodular (TUM). \docupdate{Our reasoning for this is that (i) the constraints of \docupdate{ILP}-0 is TUM (following the proof from \cite{rameshOptimalPartitioningNonConvex2022a}), and (ii) perfect matching on a bipartite graph (e.g. graph partitioned as left endpoints ($y_L^i$ variables) and right endpoints ($y_R^i$ variables)) can be solved using a linear program (LP) \cite{korteCombinatorialOptimizationTheory2006}.} For our experiments, we simply solve the linear relaxation as LPs can be solved in polynomial time \cite{karmarkarNewPolynomialtimeAlgorithm1984}.
\end{remark*}

\subsection{On reducing \docupdate{ILP} variables}
\label{sec:milp-reducing-variables}

In the previous section, we formulated \docupdate{ILP}-1 which solves the rank replanning step optimally. However, one drawback of \docupdate{ILP}-1 is that it adds a large number of decision variables ($\mathcal{O}(n^2)$) to the base \docupdate{ILP}, which affects the runtime of the \docupdate{ILP}. In this section, we propose another \docupdate{ILP} extension, namely \docupdate{ILP}-2, that uses fewer decision variables and can be more successfully used for large environments. \docupdate{Instead of matching rank endpoints to count the number of new ranks $m_\text{new}$, we formulate a bound for $m_\text{new}$ using the \textit{number of rank endpoints} that have changed. Specifically, we look to identify and count two types of changes: (i) \textit{added} endpoints and (ii) \textit{extended} endpoints (i.e. endpoints removed when ranks from the initial solution are merged). Fig. \ref{fig:all-rank-changes} illustrates cases where the ranks are changed with the introduction of obstacles and the corresponding added and extended rank endpoints.}

We first analyze the horizontal ranks to present our approach and extend the same approach to vertical ranks. Consider a binary vector $\boldsymbol{\Bar{y}_H} \in \{0,1\}^{2n}$ that indicates if a cell was a horizontal (left or right) endpoint for a rank in the initial path. We take this vector as an input to the \docupdate{ILP}. We look to compare vectors $\boldsymbol{y_H}$ and $\boldsymbol{\Bar{y}_H}$ to identify added and extended endpoints.

An \textit{added} endpoint represents an endpoint in $\boldsymbol{y_H}$ but not in $\boldsymbol{\Bar{y}_H}$. This helps count the first four of the five possible rank changes illustrated in Fig. \ref{fig:all-rank-changes}, as each rank after the change has at least one added endpoint. Let $\alpha_H \in \mathbb{R}$ be the number of added horizontal endpoints, which we compute as follows:
\begin{align*}
\vspace{-10mm}
    \alpha_H = (\boldsymbol{1} - \boldsymbol{\Bar{y}_H})^T \boldsymbol{y_H}.
\end{align*}

However, in the case of a merged rank (Fig. \ref{fig:all-rank-changes}), the new rank may not contain any added endpoints. To account for this, we also count \textit{extended} endpoints that count cells that were endpoints in $\boldsymbol{\Bar{y}_H}$ and are not endpoints in $\boldsymbol{y_H}$ \textit{while still maintaining the same coverage orientation} (horizontal in this case). Suppose a cell $c_i$ was a left endpoint in $\boldsymbol{\Bar{y}_H}$, i.e. $\Bar{y}_L^i = 1$. The cell is considered an extended endpoint if $y_L^i = 0$ and the cell is still covered horizontally ($x_H^i = 1$). We use the following piece-wise equation for cell $c_i$ to capture this relationship.
\begin{align*}
\vspace{-10mm}
    \quad \Bar{y}_L^i(1 - y_L^i) + (x_H^i - 1) = \begin{cases*} 
    \hspace{1mm} 1 & \parbox[t]{.2\textwidth}{\raggedright if $c_i$ is an extended \\ left endpoint} \\[0.1ex]
    \hspace{1mm} \leq 0 & otherwise
    \end{cases*}
\end{align*}
Now, let $\boldsymbol{\beta_H}$ be a vector of binary variables that represents whether each cell is an extended horizontal endpoint. We use this to compute a value $\beta_H$ that counts the number of extended horizontal endpoints as follows:
\begin{align*}
\vspace{-10mm}
    &\beta_H = \boldsymbol{1}^T \boldsymbol{\beta_H}, \\
    \text{where } &\boldsymbol{\beta_H} \geq \boldsymbol{\Bar{y}_H} \circ (\boldsymbol{1} - \boldsymbol{y_H}) + (
    \begin{bmatrix}
        \boldsymbol{x_H}^T & \boldsymbol{x_H}^T
    \end{bmatrix}^T - \boldsymbol{1}), \\
    &\boldsymbol{\beta_H} \in \{0,1\}^{2n},
\end{align*}
where the $\circ$ operator is the element-wise product of two vectors. \docupdate{The endpoint changes that remain to be counted are the cells that were endpoints in $\boldsymbol{\Bar{y}_H}$ and are not endpoints in $\boldsymbol{y_H}$ \textit{while changing coverage orientations} (to vertical in this case). However, such a cell will now be a part of a vertical rank that would either have an added or extended \textit{vertical} rank endpoint. Therefore, we do not count these endpoints as it tracks rank changes that are already accounted for.}

We extend these concepts to the vertical case to count all endpoint changes. Let $\boldsymbol{\Bar{y}_V}$ be a binary vector indicating whether a cell was a vertical endpoint for a rank in the initial path. The total number of added endpoints $\alpha$ and extended endpoints $\beta$ is then obtained as follows:
\begin{align}
    &\alpha = (\boldsymbol{1} - \boldsymbol{\Bar{y}_H})^T \boldsymbol{y_H} + (\boldsymbol{1} - \boldsymbol{\Bar{y}_V})^T \boldsymbol{y_V} \label{eq:alpha_final} \\
    &\beta = \boldsymbol{1}^T \boldsymbol{\beta_H} + \boldsymbol{1}^T \boldsymbol{\beta_V}, \label{eq:beta_final} \\
    \text{where } &\boldsymbol{\beta_H} \geq \boldsymbol{\Bar{y}_H} \circ (\boldsymbol{1} - \boldsymbol{y_H}) + (
    \begin{bmatrix}
        \boldsymbol{x_H}^T & \boldsymbol{x_H}^T
    \end{bmatrix}^T - \boldsymbol{1}), \label{eq:beta_h_cons} \\
    &\boldsymbol{\beta_V} \geq \boldsymbol{\Bar{y}_V} \circ (\boldsymbol{1} - \boldsymbol{y_V}) + (
    \begin{bmatrix}
        \boldsymbol{x_V}^T & \boldsymbol{x_V}^T
    \end{bmatrix}^T - \boldsymbol{1}), \label{eq:beta_v_cons} \\
    &\boldsymbol{\beta_H}, \boldsymbol{\beta_V} \in \{0,1\}^{2n}. \label{eq:beta-bounds}
\end{align}

\begin{figure}
    \centering
    \includegraphics[width=0.85\linewidth]{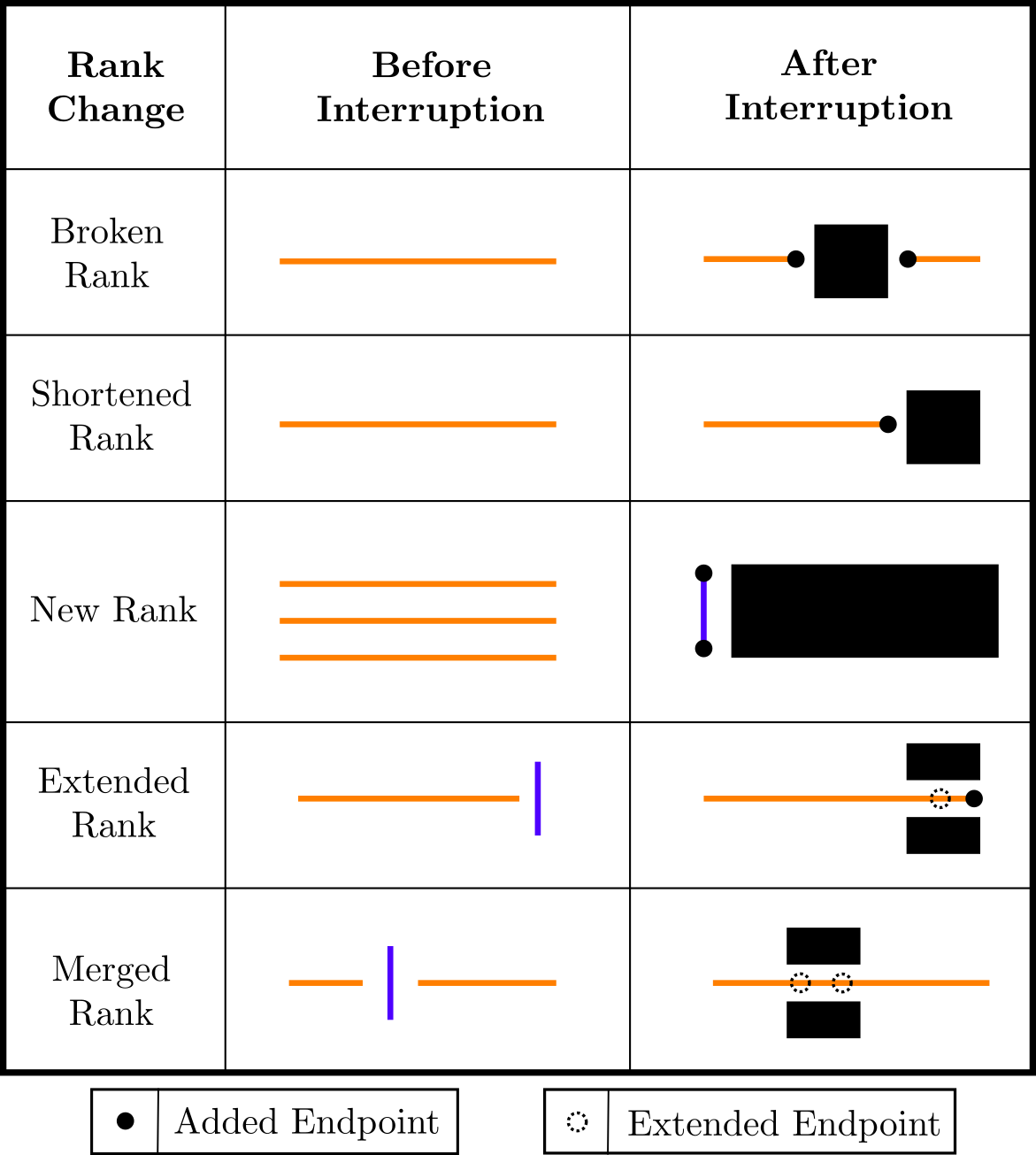}
    \caption{\docupdate{Illustration of rank endpoint changes as a result of rank replanning. Each row shows horizontal (orange lines) and vertical (purple lines) ranks blocked by obstacles (black boxes) and the corresponding replanned ranks and endpoint changes.}}
    \label{fig:all-rank-changes}
\end{figure}

We now look to bound the number of new ranks $m_{\text{new}}$ using $\alpha$ and $\beta$, for which we exploit the following property.
\begin{lemma}
Consider $\alpha$ and $\beta$ to be the number of added and extended endpoints respectively for a set of replanned ranks. Then the number of new ranks $m_{\text{new}}$ in the set is bounded as follows:
\begin{align*}
    m_{\text{new}} \leq \alpha + \beta/2.
\end{align*}
\label{lem:endpoints}
\end{lemma}
\begin{proof}
    We start by considering all possible changes to the ranks as illustrated in Fig. \ref{fig:all-rank-changes}. For the first four possible changes, the minimum number of added endpoints for each new rank is 1. However, for a merged rank, the number of added endpoints for the new rank is 0, so we instead count the minimum number of extended endpoints. Merging two ranks can be thought of as extending two endpoints so that the ranks join to form a new rank. Therefore, a merged rank results in at least two extended endpoints. Note that merging more than 2 ranks to form one new rank results in more than 2 extended endpoints. As a result, the maximum number of merged ranks is $\beta/2$. With this, we have the following bound on $m_{\text{new}}$:
     \[
     m_{\text{new}} \leq \alpha + \beta / 2,
     \]
     which proves the lemma.
\end{proof}

We now propose our second \docupdate{ILP} extension. We use the bound we derived in Lemma \ref{lem:endpoints} to constrain $m_{\text{new}}$ in the \docupdate{ILP}. Following Proposition \ref{prop:optimality}, we add an additional term $\epsilon \text{ } (\alpha + \beta/2)$ to the objective function, where $\epsilon = 1/(\Bar{m} + 1) $, to compute a solution with the smallest $m_\text{new}$ without dominating the minimization of $m$. The resulting \docupdate{ILP} is as follows:
\begin{align}
\nonumber \text{\textbf{\docupdate{ILP}-2:}} \hspace{5mm} & \\
\min_{\boldsymbol{x_H}, \boldsymbol{x_V}, \boldsymbol{y_H}, \boldsymbol{y_V}} \quad & m  + \epsilon \text{ } (\alpha + \beta/2) \label{eq:new-obj-milp} \\
\textrm{s.t.} \quad 
& \nonumber \text{Eqs. (\ref{eq:c1-new}), (\ref{eq:c2-new}), (\ref{eq:c3}), (\ref{eq:alpha_final}) - (\ref{eq:beta-bounds})} \\
& \alpha + \beta/2 \leq \Bar{m} \label{eq:endpoint-constraint} \\
& \boldsymbol{x_H}, \boldsymbol{x_V} \in \{0, 1\}^{n} \\
& \boldsymbol{y_H}, \boldsymbol{y_V} \in \{0, 1\}^{2n}.
\end{align}

Solving \docupdate{ILP}-2 guarantees a solution where $m_\text{new} \leq \Bar{m}$. Moreover, \docupdate{ILP}-2 only adds $\mathcal{O}(n)$ decision variables to the base \docupdate{ILP} ($\boldsymbol{\beta_H}$ and $\boldsymbol{\beta_V}$). So, we expect \docupdate{ILP}-2 to run faster than \docupdate{ILP}-1. We observed that Remark \ref{remark:linear-relaxation} also applies for \docupdate{ILP}-2, so we simply solve the linear relaxation to obtain integral solutions. Since there are negative terms involving $\boldsymbol{y}_H$ and $\boldsymbol{y}_V$ (Eqs. (\ref{eq:beta_h_cons}), (\ref{eq:beta_v_cons}), \docupdate{the following constraint in the linear relaxation ensures that the optimization problem is bounded}: 
\[\boldsymbol{y}_H, \boldsymbol{y}_V \leq \boldsymbol{1}.\]

Unlike \docupdate{ILP}-1, \docupdate{ILP}-2 does not guarantee an optimal solution to the rank replanning problem, as we use a pessimistic bound on $m_\text{new}$ in our \docupdate{ILP} constraint. However, we observe in our simulation results that the paths replanned with ranks from \docupdate{ILP}-2 are close in cost to those from \docupdate{ILP}-1.
\section{The TOUR-REPLAN Subroutine}
\label{sec:rank-touring}

\begin{figure}
    \centering
    \includegraphics[width=0.9\linewidth]{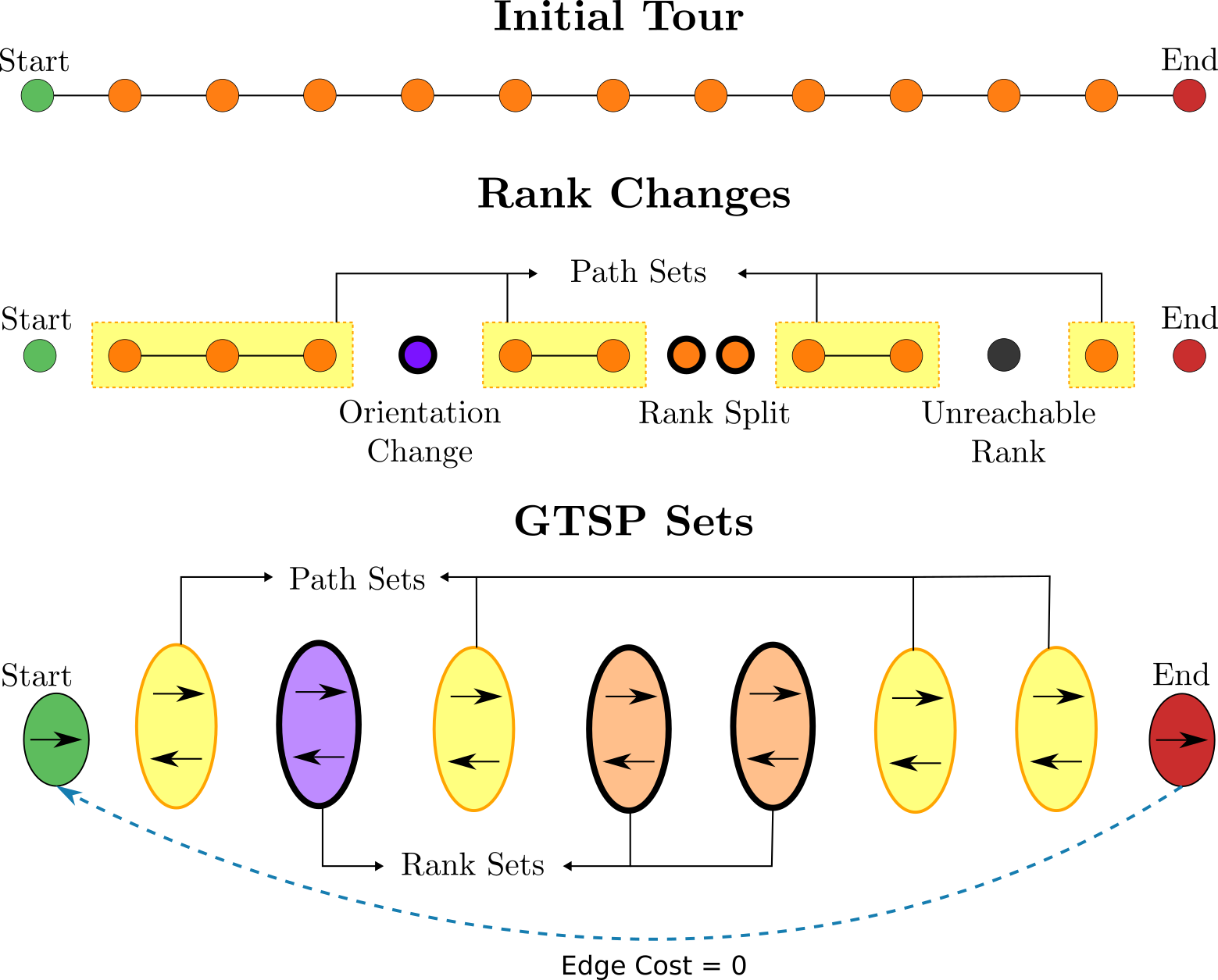}
    \caption{All possible changes to an example initial tour of only horizontal ranks (orange circles). The initial tour requires replanning if the ranks (i) switch orientations to form vertical ranks (purple circles), (ii) are split by obstacles, and (iii) become unreachable due to obstacles (black circle). The resulting GTSP sets are shown consisting of (i) path sets (yellow), (ii) rank sets for each new replanned rank, and (iii) artificial vertices representing the start and end of the initial path, with a directed zero-cost edge connecting them.}
    \label{fig:tour-changes}
\end{figure}

In this section, we describe our approach to compute a tour of the replanned ranks to obtain our new coverage path. Similar to the approaches from \cite{bochkarevMinimizingTurnsRobot2016} and \cite{rameshOptimalPartitioningNonConvex2022a}, we formulate this as a Generalized Travelling Salesman Problem (GTSP) on a \docupdate{fully-connected} auxiliary graph representing all connections between ranks. Each set in the graph represents a rank and consists of two vertices to represent the directions in which the rank can be traversed. We call these sets as \textit{rank sets}. \docupdate{The edge costs are the \textit{transition times} between two ranks, i.e. the time for the robot to travel from one rank endpoint to another. Since the transition times depend on the robot's design and dynamics, we compute the edge costs by assuming dynamics similar to \cite{rameshOptimalPartitioningNonConvex2022a} and \cite{vandermeulenTurnminimizingMultirobotCoverage2019}, where (i) the robot traverses a straight line with piecewise constant acceleration with a maximum velocity and (ii) the robot stops and turns in place with a constant angular velocity. Solving GTSP on this graph obtains a tour that covers each rank while minimizing the transition times between the ranks. In Section \ref{sec:results-setup}, we provide more details on the robot dynamics parameters used to compute transition times in our simulations.}

However, we look to address two additional challenges in the tour replan. Firstly, we look to limit the size of the GTSP instance by considering only the interrupted sections of the initial path. Fig \ref{fig:tour-changes} gives an illustration of possible path interruptions caused by (i) ranks changing orientations (ii) ranks split by an obstacle, and (iii) unreachable ranks. The ranks that are unchanged correspond with sub-sections of the initial path that are unchanged. We retain each such section by representing it as its own GTSP set, consisting of two vertices representing possible traversal directions. We refer to this set as a \textit{path set}. The GTSP is computed on a graph containing these path sets and rank sets that each represent a new rank obtained from RANK-REPLAN to be added to the tour.

The second challenge is that we must also take into account the position of the robot and the desired end of the tour. To address this, we add artificial start and end vertices to the GTSP graph, as in Fig. \ref{fig:tour-changes}. To ensure that the solution is a path from the start vertex to the end vertex, we add a directed zero-cost edge from the end vertex to the start vertex. Solving the GTSP on this graph gives us the replanned coverage path.

\docupdate{In this paper, we created a fully-connected auxiliary graph. There are alternative approaches where the auxiliary graph need not be fully-connected, such as those using lazy edge evaluations to compute a tour \cite{sahaPlanningToursRobotic2006}. However, the above challenges must still be addressed for such a formulation using the approach described for the fully-connected case.}
\section{Simulation Results}
\label{sec:results}

In this section, we evaluate the OARP-Replan approach in simulation using real-world maps. \docupdate{We compare the performance of OARP-Replan in online coverage replanning scenarios against greedy detour (GD) replan and other planning approaches \cite{rameshOptimalPartitioningNonConvex2022a, vandermeulenTurnminimizingMultirobotCoverage2019, kusnurCompleteDecompositionFreeCoverage2022}. We also compare OARP-Replan with offline coverage planning approaches provided with complete obstacle information \cite{bahnemannRevisitingBoustrophedonCoverage2021, agarwalAreaCoverageMultiple2022a}, where OARP-Replan is given no prior obstacle information.} Finally, we demonstrate using OARP-Replan to replan coverage paths in a high-fidelity ROS simulation.

\subsection{Simulation setup}
\label{sec:results-setup}

We simulate the robot covering partially known environments: the base environment is known but there are unknown obstacles that the robot discovers during coverage using its sensor. We plan initial coverage paths for the base environments using the OARP method \cite{rameshOptimalPartitioningNonConvex2022a} and replan the path as new obstacles are detected.

\begin{table}
\vspace{6pt}
\centering
\caption{Robot parameters used for experiments}
\label{table:parameters}
\begin{tabular}{@{} l l@{}}
\toprule
Parameter & Value \\
\midrule
Coverage tool width & 0.8 $m$\\
Maximum linear velocity & 1 $m/s$\\
Linear acceleration & $\pm 0.5$ $m/s^2$ \\
Angular velocity &$30^{\circ}/s$ \\
Radial sensor range & 5.6 $m$\\
\bottomrule
\end{tabular}
\end{table}

\emph{\textbf{Dataset:}} We use a dataset of \docupdate{8} real-world 2D maps obtained from Avidbots as our base environments, with their minimum coverage path lengths ranging from 145 $m$ to 1900 $m$. We then add randomly generated obstacles with varying sizes and orientations to create a set of unknown environments. Specifically, we create two datasets of unknown environments: (i) \textit{fixed clutter} dataset where the new obstacles occupy 10\% of the base map's free area for all maps, and (ii) \textit{varying clutter} dataset where starting with a candidate base environment, we generate unknown environments by varying the \% area occupied by new obstacles. For the varying clutter dataset, we chose the base environment to be the one used in Fig. \ref{fig:replanning-example}(a) and we test with 2 to 20 \% area occupied. Due to the confidentiality of the environments, we are unable to share the complete dataset, but we have created anonymized environments for visualization, as seen in Figs. \ref{fig:replanning-example}, \ref{fig:replan-results}, and \ref{fig:ros-simulation}. Fig. \ref{fig:replan-results} shows an example base environment containing three unknown obstacles (red boxes) for which the initial coverage path (dark path) needs replanning.

\emph{\textbf{Implementation:}} We test the two variants of OARP-Replan (Algorithm \ref{alg:oarp-replan}) based on which RANK-REPLAN method is used: (i) \textit{OARP-Replan-1} which uses \docupdate{ILP}-1, and (ii) \textit{OARP-Replan-2} which uses \docupdate{ILP}-2. Both approaches were implemented in Python. For each \docupdate{ILP}, we solve the corresponding linear relaxation using the free SCIP solver \cite{achterbergSCIPSolvingConstraint2009}. Then, we compute a tour of the replanned ranks by solving the GTSP formulation from Section \ref{sec:rank-touring}, for which we use the open-source GLNS solver \cite{smithGLNSEffectiveLarge2017}. Table \ref{table:parameters} shows the robot parameters used to compute the transition time between ranks. The robot is also equipped with a radial sensor in simulation (radius in Table \ref{table:parameters}) that detects new obstacles and runs OARP-Replan.

\emph{\textbf{Metrics:}} \docupdate{Our main metric of interest is the time taken by the robot to complete coverage, i.e. \textit{total coverage time}. This is computed as a sum of drive time and \textit{stoppage time}, where stoppages occur when the robot is waiting for a new path to be computed. Table \ref{table:parameters} shows the robot parameters used to compute the path drive time using the dynamics model from Section \ref{sec:rank-touring}. To measure stoppage, we compute the minimum excess time taken by the algorithm to compute the replanned path as follows:
\begin{align*}
    \nonumber\text{Minimum stoppage time} &= \max(\{0,  
    [\text{Time to plan transitions}] \\ 
    & \quad + [\text{Time to find GTSP tour}]\\ 
    & \quad - [\text{Time to obstacle encounter}]\}. \label{eq:stoppage-time-comp}
\end{align*}}

\begin{figure}
    \centering
    \includegraphics[width=0.75\linewidth]{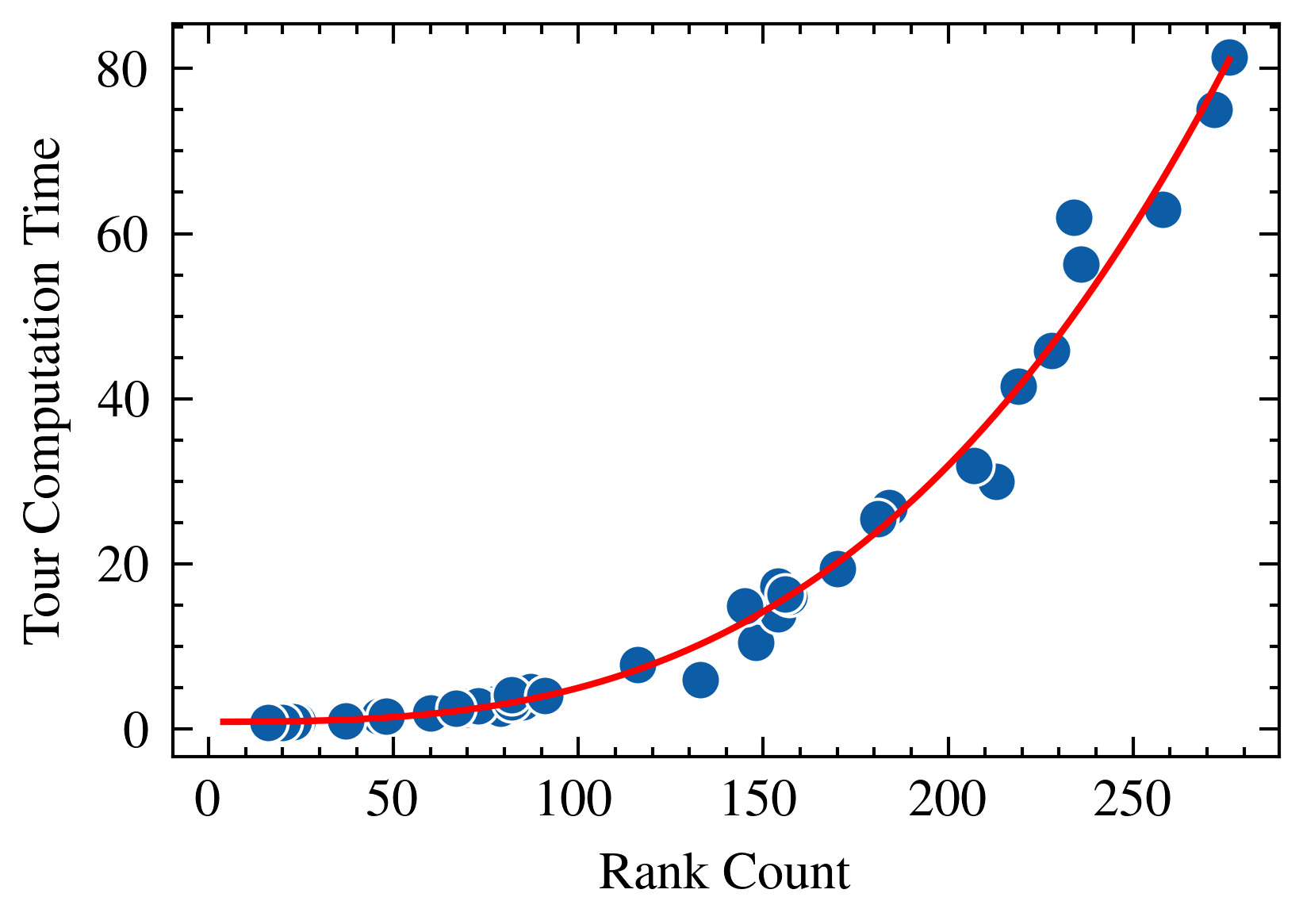}
    \caption{Data used to estimate TOUR-REPLAN runtime, where we plot the number of ranks (x-axis) and the time taken to compute all transitions and replan GTSP (y-axis). \docupdate{The red line shows the cubic polynomial fit to the data.}}
    \label{fig:GTSP-estimator plot}
\end{figure}

\emph{\textbf{Replan runtime estimator:}} \docupdate{In each simulation run, the robot conducts many replans of the coverage path depending on the obstacles it observes. For each replan, we compute a time $\tau$ to reach the next obstacle encounter along the current path using the robot model from Section \ref{sec:rank-touring}. Using $\tau$, we then determine a budget $\Bar{m}$ for rank replanning using an estimator function $\Hat{T}(m)$ of the runtime needed to compute a tour for $m$ ranks. The inverse of $\Hat{T}(m)$ is then used to compute the replanning budget, i.e. $\Bar{m} = \Hat{T}^{-1}(\tau)$.} 

The runtime estimate $\Hat{T}(m)$ includes (i) the estimated time to compute transition paths $\Hat{T}_\text{trans}(m)$, and (ii) the estimated time to solve the GTSP $\Hat{T}_\text{GTSP}(m)$. Transition path computation is included as it is a time-intensive process, especially for large maps. Formally, 
\begin{align*}
    \Hat{T}(m) = \Hat{T}_\text{trans}(m) + \Hat{T}_\text{GTSP}(m).
\end{align*}

We use data collected from previous coverage planning runs to estimate each function. \docupdate{Our solver of choice, GLNS, typically computes a tour in $\mathcal{O}(m^2)$ time for $m$ ranks \cite{smithGLNSEffectiveLarge2017}.} Therefore, we estimate $\Hat{T}_\text{GTSP}$ by fitting a cubic polynomial through the data as shown in Fig. \ref{fig:GTSP-estimator plot}, \docupdate{where we plot the time to compute the tour for different rank counts ($m$ values).} For $\Hat{T}_\text{trans}(m)$, we use the following map-specific linear estimator. Let $t_\text{avg}$ be the average time to compute the transition paths for one rank in the given map. We compute $\Hat{T}_\text{trans}(m)$ as follows:
\begin{align*}
    \Hat{T}_\text{trans}(m) = t_\text{avg} m.
\end{align*}
\docupdate{Under-estimating the replan runtime will result in robot stoppage, as the robot must wait for the new path to be computed. Using the above estimator with OARP-Replan in our simulations, we observed little to no such robot stoppage. Therefore, we use this simple estimator for the rest of the paper as it works well for OARP-Replan. Regardless, this estimator may be inaccurate in cases where the time to compute transition path differs significantly from $t_\text{avg}$ (e.g. time to compute transitions increase when a large number of obstacles are added to an environment). While other sophisticated estimators can be used with OARP-Replan to handle this, we do not focus on this for this paper.}

\subsection{Comparisons with online approaches}
\label{sec:results-cost-comp}

\docupdate{
In this section, we compare the total coverage time (drive time $+$ stop time) of our two proposed OARP-Replan approaches against online coverage planning/replanning approaches. Our baseline approaches are as follows: (i) the OARP approach as proposed in \cite{rameshOptimalPartitioningNonConvex2022a}, (ii) the heuristic (HR) approach from \cite{vandermeulenTurnminimizingMultirobotCoverage2019}, (iii) the online search-based (SB) coverage planning approach from \cite{kusnurCompleteDecompositionFreeCoverage2022}, and (iv) GD replan, i.e. the practical case of following a turn-minimizing coverage path with a local planner. The first two approaches minimize turns in the path but they are offline planners. So we modify them to be usable online for replanning as follows:
\begin{enumerate}[(i)]
    \item while the robot is following the path and detects new obstacle encounters, the planning approach is used to replan the ranks for the remaining coverage path, and
    \item the tour of the resulting ranks is computed by solving the GTSP as formulated in Section \ref{sec:rank-touring}, where unchanged parts of the coverage path are preserved.
\end{enumerate}
For OARP, we replan coverage ranks by solving the linear relaxation of \docupdate{ILP}-0 from Section \ref{sec:rank-replan} (identical to the LP in \cite{rameshOptimalPartitioningNonConvex2022a}). We refer to this online variant of OARP as \textit{OARP-Original}. For the HR approach, we solve Algorithm 1 from \cite{vandermeulenTurnminimizingMultirobotCoverage2019} to replan ranks. The latter two approaches (SB and GD) are designed for online coverage planning with fast computation of coverage paths. As a result, these approaches have negligible stoppage time. The online SB approach plans the entire coverage path online and does not require an initial path. However, the SB approach minimizes path length rather than total coverage time with turns.
}

\begin{figure}
    \centering
    \subcaptionbox{}{
    \centering
    \includegraphics[width=0.465\linewidth]{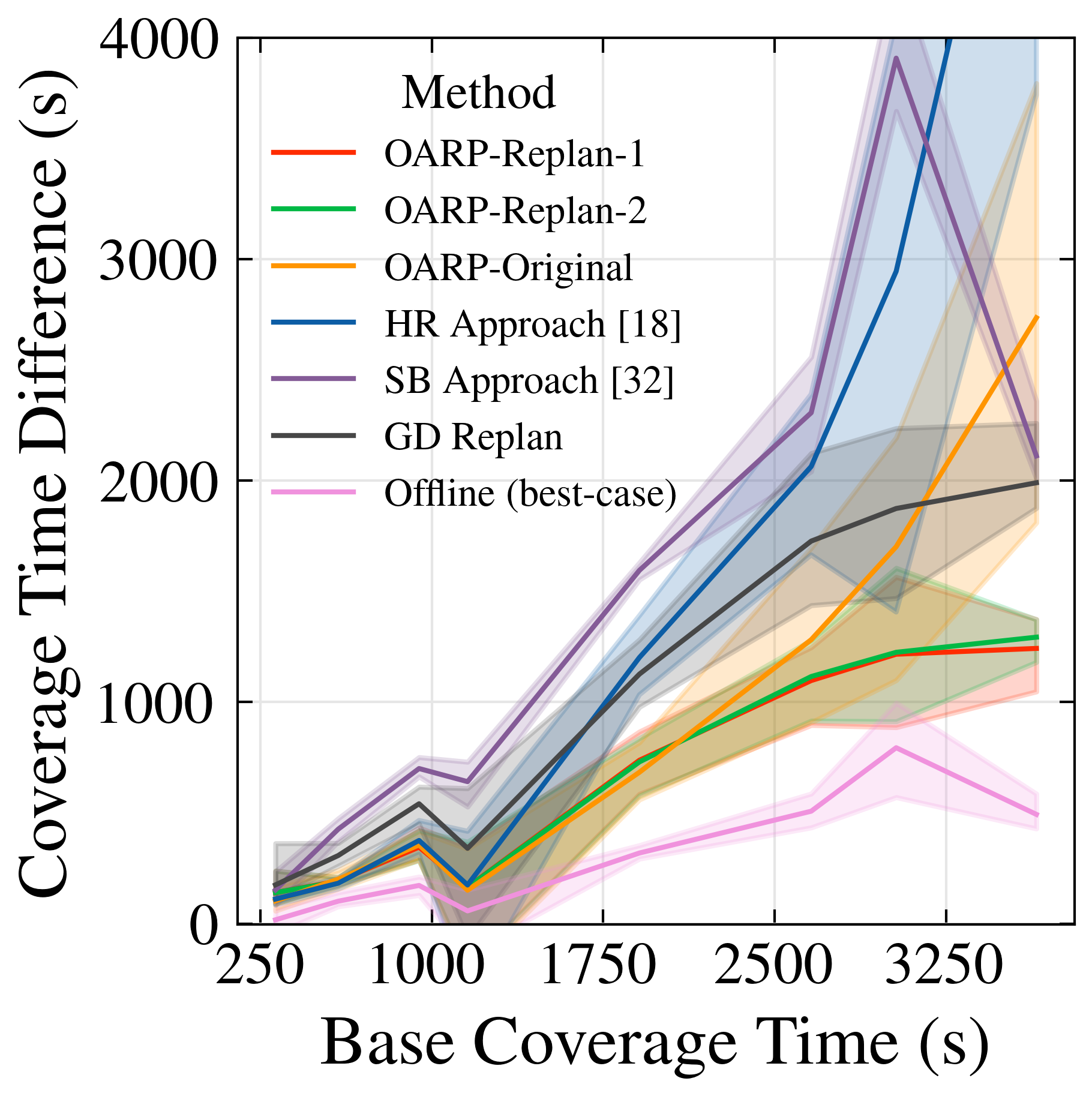}
    }
    \subcaptionbox{}{
    \centering
    \includegraphics[width=0.465\linewidth]{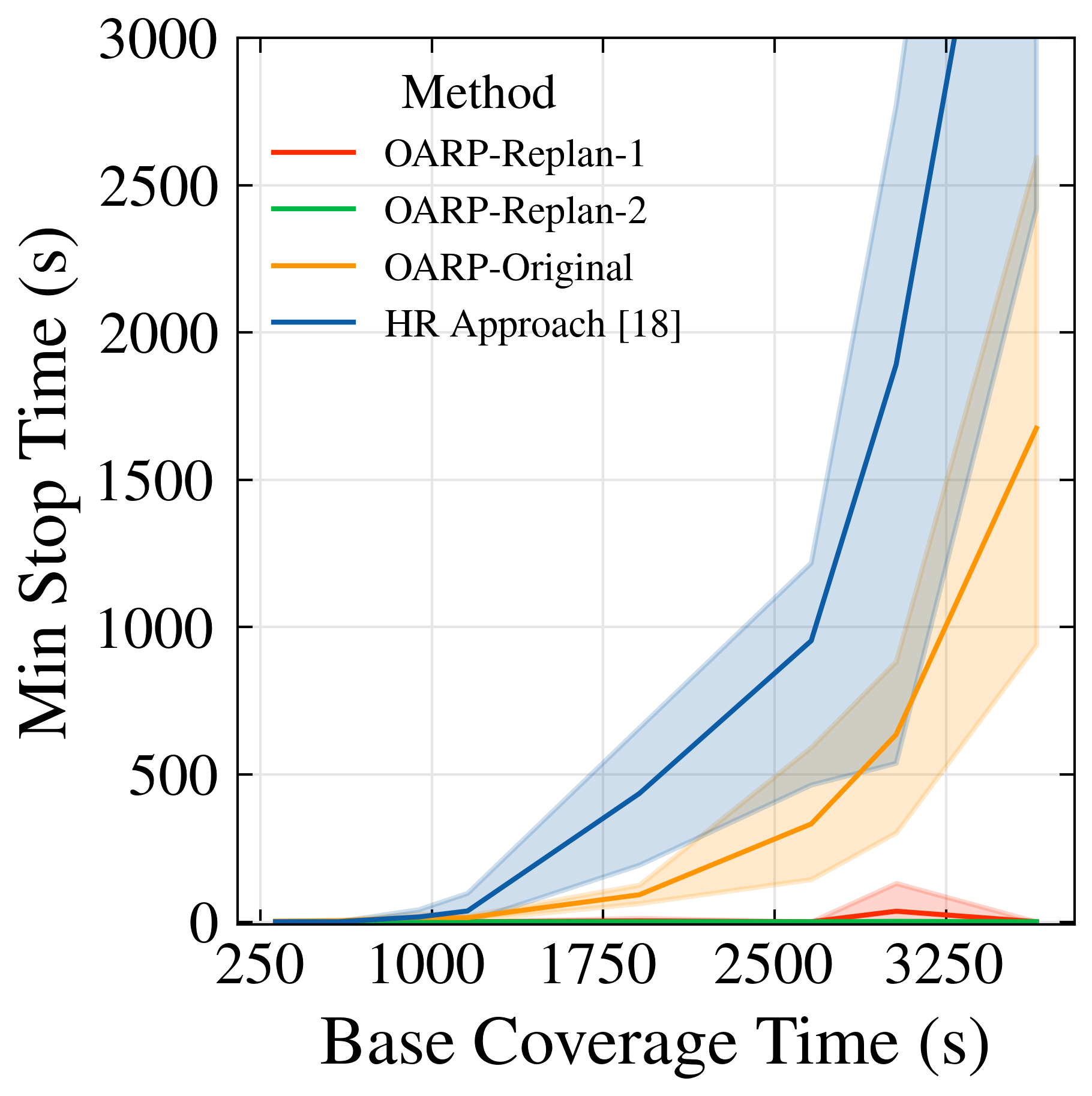}
    }
     \caption{\docupdate{Comparison of the online replanning approaches using the fixed clutter dataset: (a) Difference between total coverage time and base coverage time. (b) Minimum stop time of the robot during coverage (GD and SB omitted as replanning time is negligible). The line shows the average and the lighter region shows the data interval.}}
    \label{fig:comparison-fixed-clutter}
\end{figure}

\begin{figure*}
\centering
\subcaptionbox{\textit{GD replan} result}
{\centering
\includegraphics[width=0.24\linewidth]{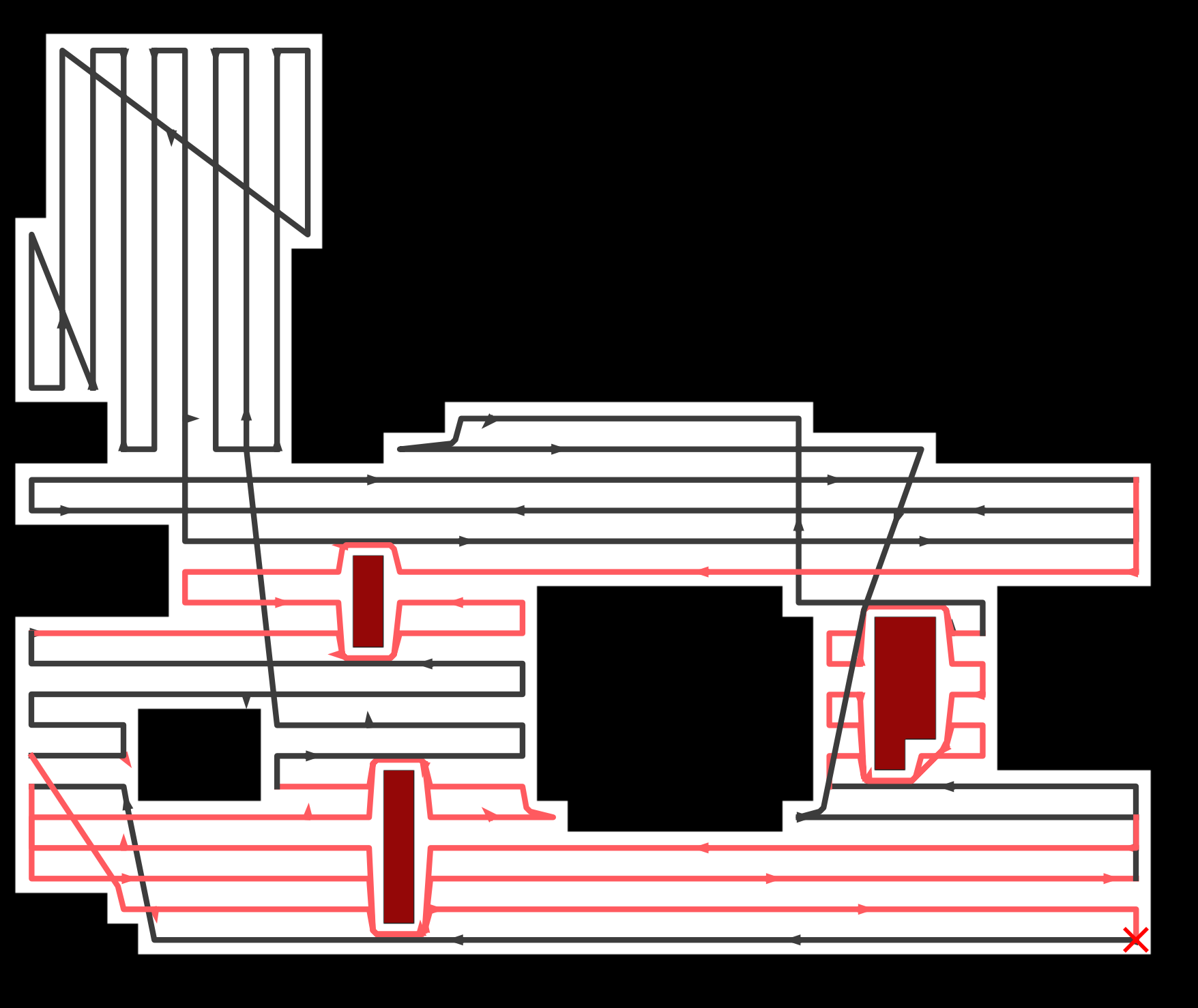}
}
\hspace{2mm}
\subcaptionbox{\textit{OARP-Replan-1} result}
{\centering
\includegraphics[width=0.24\linewidth]{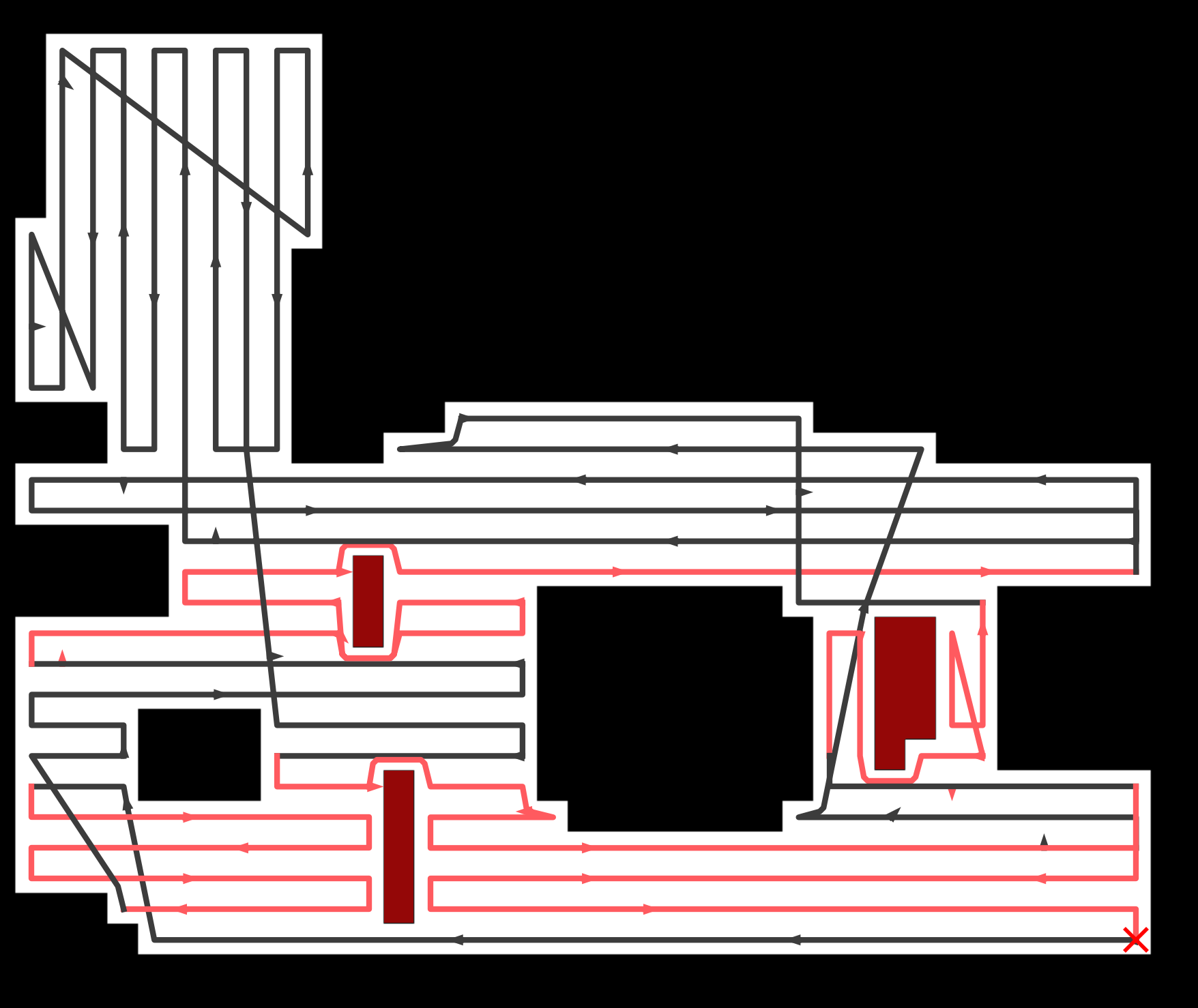}
}
\hspace{2mm}
\subcaptionbox{Path improvement examples}
{\centering
\includegraphics[width=0.3\linewidth]{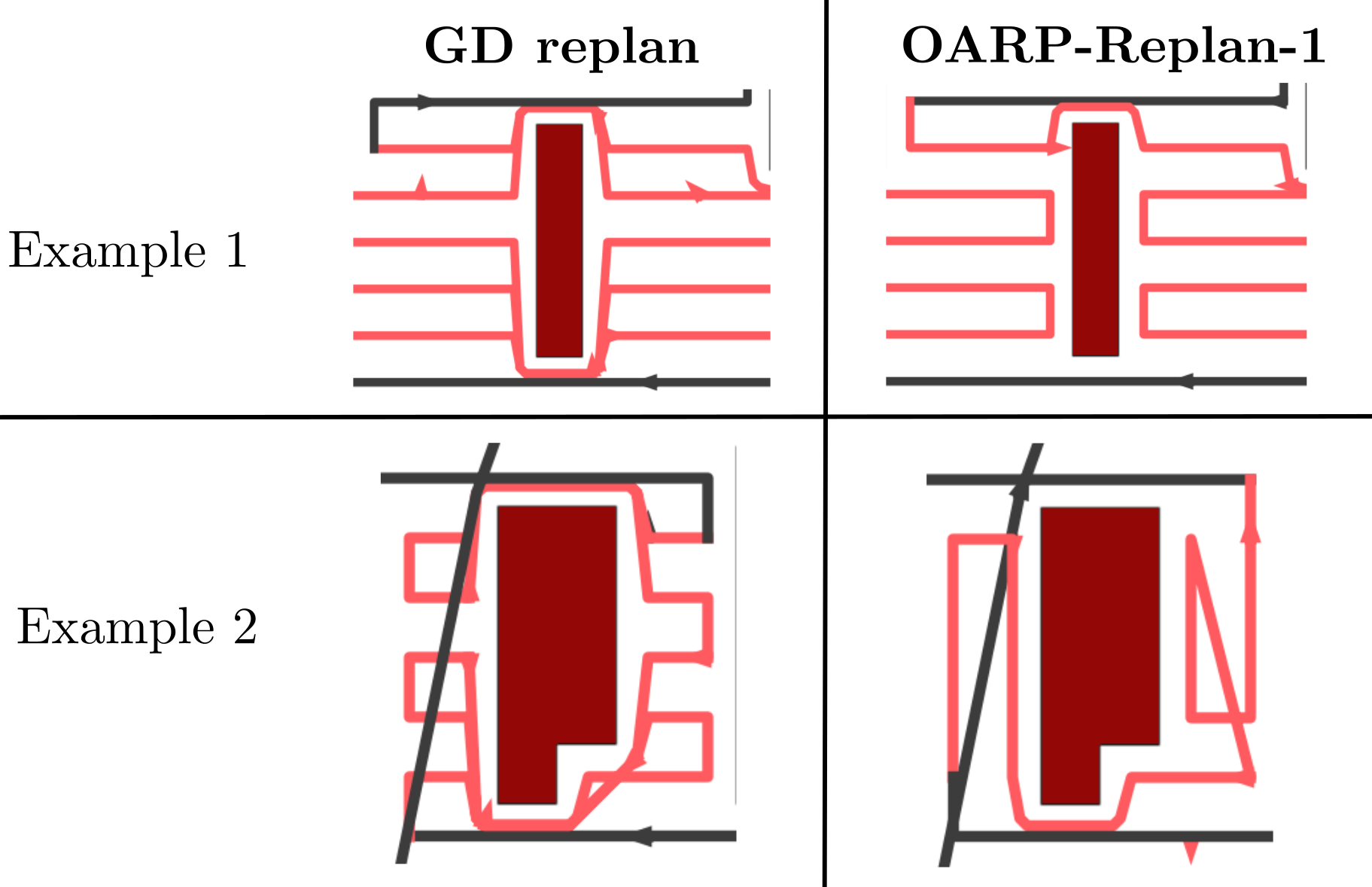}
}
\caption{Qualitative results of replanning an initial coverage path for the example environment with unknown obstacles. \vspace{-2mm}}
\label{fig:replan-results}
\end{figure*}

\emph{\textbf{Fixed Clutter Dataset:}} We first review the results for the fixed clutter dataset. For each base environment, we ran 5 trials, where each trial was generated with a new random set of obstacles occupying 10\% of the free space in the environment. \docupdate{The maps are sorted in the order of \textit{base coverage time}, which is the time to cover the base environment (without the unknown obstacles) using the initial coverage path. Fig. \ref{fig:comparison-fixed-clutter}(a) and (b) show the results of the comparison where we plot the following for each method: (a) the difference between the total coverage time (including stops) and the base coverage time (i.e. \emph{total} $-$ \emph{base}), and (b) the minimum stop time. To visualize the \textit{best-case} coverage time for each environment (a lower bound), we also include the results for an offline planner (OARP) when provided with all obstacle information a-priori. We observe that both OARP-Replan approaches result in lower total coverage times relative to all baseline approaches, and comes closest to the performance of the offline planner. In comparison with GD replan, OARP-Replan exhibits an average improvement of about 12\% across all maps.} To visualize this improvement, we show an example map with three unknown obstacles (red boxes) in Figs. \ref{fig:replan-results}(a) and (b) where the initial path was replanned using GD replan and OARP-Replan-1 respectively. Fig. \ref{fig:replan-results}(c) shows two examples where we observe qualitative differences between the approaches. In example 1, paths replanned using GD replan (red paths) have multiple long detours around the obstacle, whereas OARP-Replan-1 has no such detours. In example 2, GD replan generates messy overlapping paths when avoiding the large obstacle, while OARP-Replan-1 reorients the ranks in this region to obtain a path with fewer overlaps. Using OARP-Replan-2, we observe similar improvements in the replanned path for this example. Since the tour of the replanned ranks is computed with little to no stoppage, the replanning time does not affect the total coverage time.

\docupdate{When comparing with OARP-Original and the HR approach \cite{vandermeulenTurnminimizingMultirobotCoverage2019}, both OARP-Replan approaches achieve faster coverage times, with the gap increasing as initial paths get longer (larger maps). This is because both \docupdate{ILP}-0 and the HR approach do not limit the change to the coverage path}, leading to longer replanning times and resulting robot stoppages (Fig. \ref{fig:comparison-fixed-clutter} (b)). In contrast, OARP-Replan approaches cause little to no stoppage. We expect this difference in performance to be exaggerated in weaker computation setups, such as onboard computers on robot platforms. \docupdate{The OARP-Replan approaches also outperform the SB approach as a result of minimizing the number of ranks (turns) in the coverage path. However, we observe that the SB approach performs well on the map with the longest base coverage time when compared to the other baselines. Note that this map is not the largest in terms of area, so we reason that the performance of SB is more closely related to map area than to base coverage time (which may be influenced by complex layouts/geometry).} Between the OARP-Replan approaches, OARP-Replan-1 consistently plans paths with slightly shorter coverage times than OARP-Replan-2.

\begin{figure}
    \centering
    \subcaptionbox{}{
    \centering
    \includegraphics[width=0.465\linewidth]{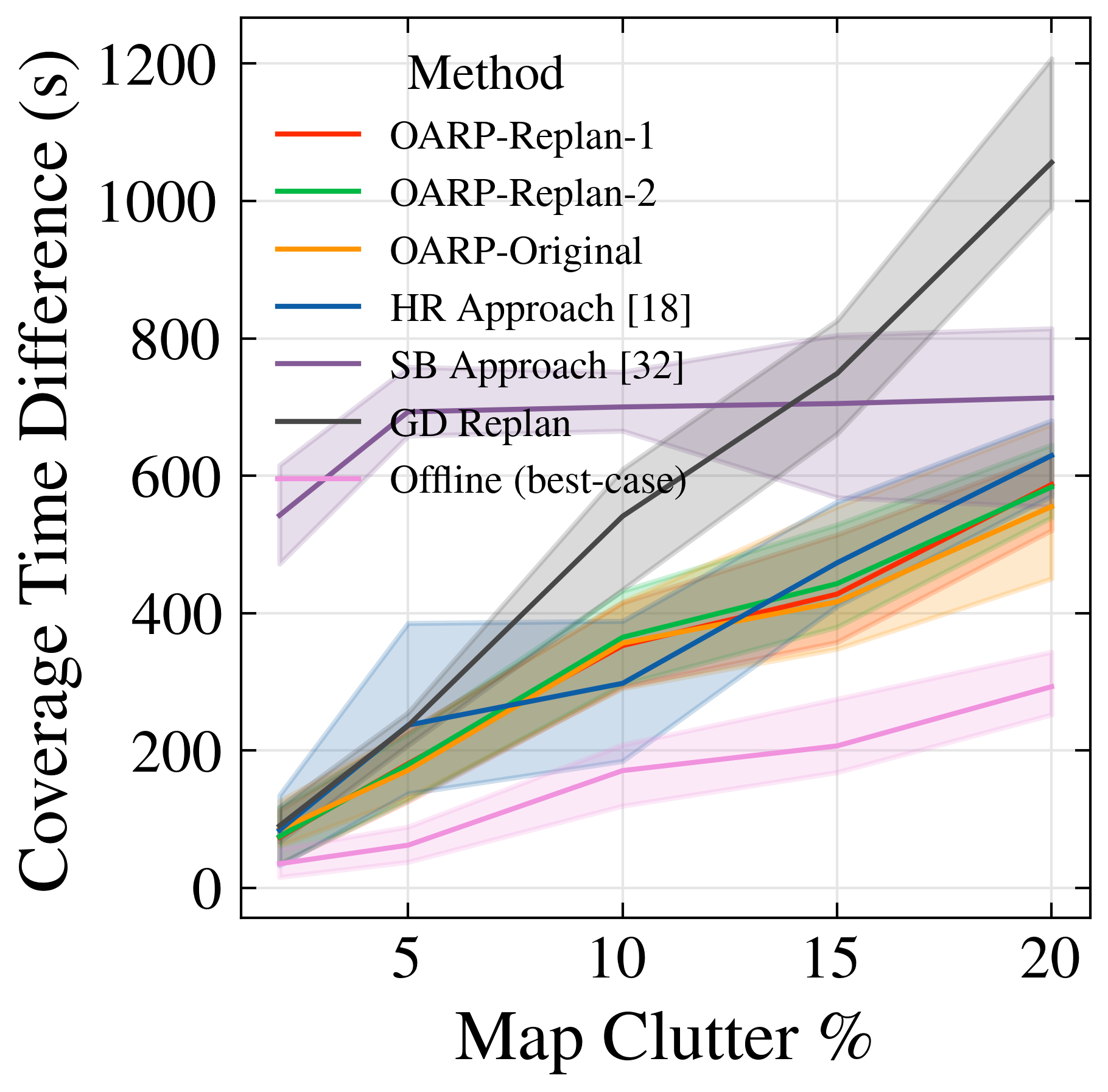}
    }
    \subcaptionbox{}{
    \centering
    \includegraphics[width=0.465\linewidth]{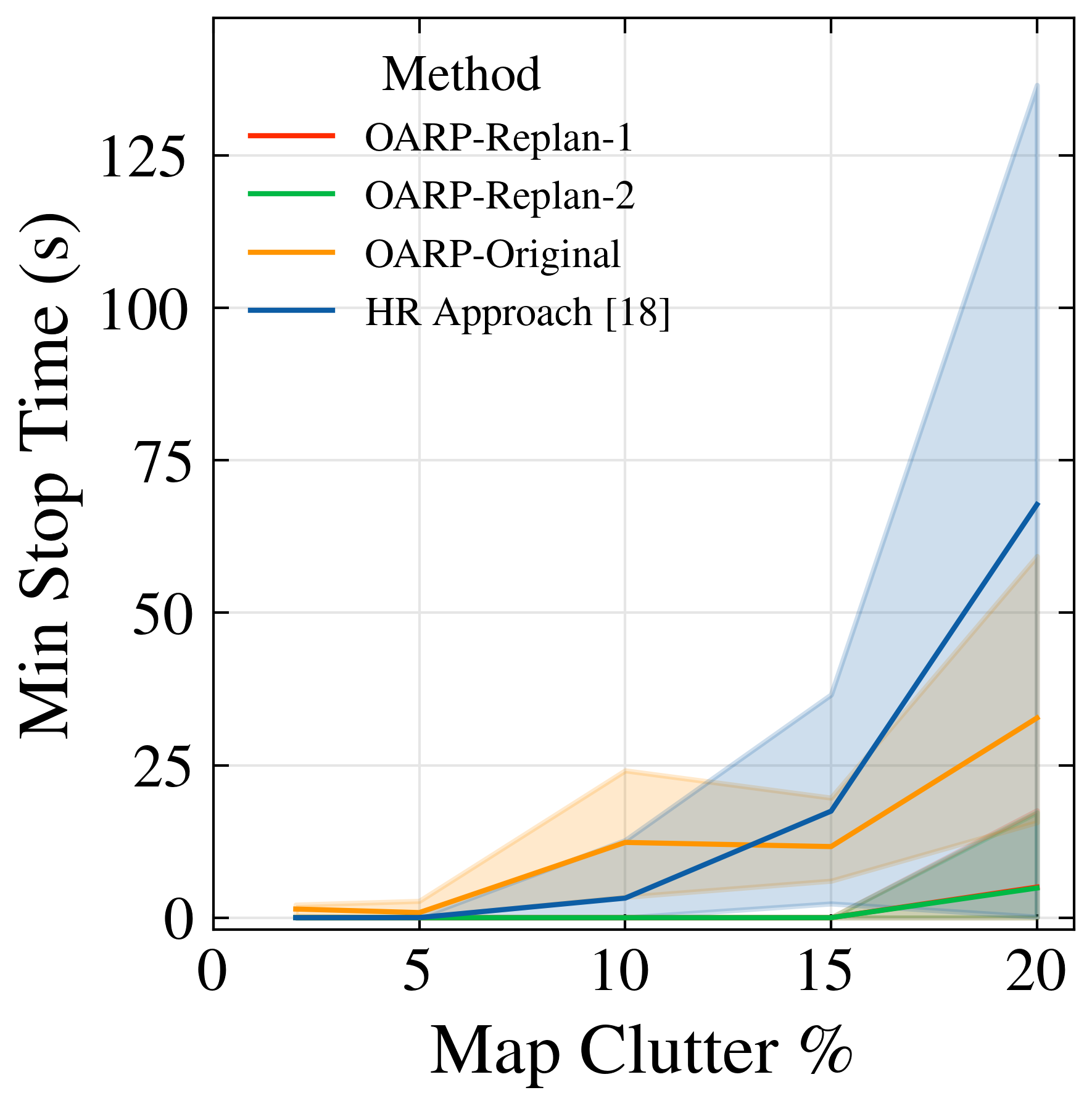}
    }
     \caption{\docupdate{Comparison of the online replanning approaches using the varying clutter dataset: (a) Difference between total coverage time and base coverage time. (b) Minimum stop time of the robot during coverage (GD and SB omitted as replanning time is negligible). The line shows the average and the lighter region shows the data interval.}}
    \label{fig:comparison-varying-clutter}
\end{figure}

\emph{\textbf{Varying Clutter Dataset:}} We also conducted similar experiments using the varying clutter dataset, where we ran 5 trials for each clutter level. \docupdate{The map chosen for this dataset has a base coverage time of 944\textit{s}. Fig. \ref{fig:comparison-varying-clutter} shows the results of the coverage time comparisons. In Fig. \ref{fig:comparison-varying-clutter}(a), we plot the difference between the total coverage time and base coverage time.} We observe that the OARP-Replan approaches outperform GD replan by 18\% on average, with the performance gap increasing with clutter level. This is because, with more obstacles, GD replan causes many detours, while the OARP-Replan approaches tend to shorten such detours.

\docupdate{In comparison with OARP-Original and the HR approach \cite{vandermeulenTurnminimizingMultirobotCoverage2019}, the OARP-Replan approaches plan coverage paths with comparable coverage times. However, both OARP-Original and HR experience longer stops with increasing clutter (Fig. \ref{fig:comparison-varying-clutter}(b)).} We also observe some stoppage with using the OARP-Replan approaches in environments of 20\% clutter. We believe this is because we do not update our runtime estimator during an experiment, making it inaccurate for highly cluttered environments (e.g. computing transitions take more time with more obstacles). A more sophisticated estimator that updates itself during coverage may eliminate this stoppage. \docupdate{We also observe that the coverage time using the SB approach plateaus with increasing clutter due to SB's rapid online computation of coverage paths. Although OARP-Replan outperforms SB at all the clutter levels tested, this performance gap may decrease with higher clutter levels ($> 20\%$).} 

\subsection{\docupdate{ILP} runtime comparisons}
\label{sec:results-runtime}

In this section, we compare the average runtimes of the rank replanning \docupdate{ILP}s for both OARP-Replan approaches. Our approach relies on the \docupdate{ILP} being solved efficiently so that the GTSP tour can be reliably computed within the time budget (preventing robot stoppages). The \docupdate{ILP} from OARP-Original is excluded from this comparison as it does not constrain the GTSP input size.

Fig. \ref{fig:comparison-results-runtime}(a) and (b) show the comparison results in the fixed clutter and varying clutter datasets respectively, where we plotted the average \docupdate{ILP} runtimes for each map. \docupdate{For the fixed clutter dataset, the maps are sorted in the order of increasing size (area) of the environment, and in-turn, the number of grid cells in the IOP.} We observe that \docupdate{ILP}-2's runtime scales better than \docupdate{ILP}-1 as the environment area increases. This is mainly due to the reduction in the \docupdate{ILP} variables that speed up runtime. Moreover, we also observe in Fig \ref{fig:comparison-fixed-clutter}(b) that using \docupdate{ILP}-2 resulted in slightly shorter stop times for larger maps (the plot for OARP-Replan-2 is close to zero while OARP-Replan-1 goes up to 100$s$ for one map). We believe that this difference will be exaggerated when computation power is constrained (e.g. low-cost robot hardware) to the point that it will affect coverage time. Therefore, for such applications, we recommend \docupdate{ILP}-2 over \docupdate{ILP}-1. A future improvement of this work could involve combining both \docupdate{ILP}s, with \docupdate{ILP}-2 as a fallback for \docupdate{ILP}-1.

For the varying clutter dataset, we observe that the average runtimes of both \docupdate{ILP}s are comparable, with \docupdate{ILP}-1 performing better overall. Increasing clutter seems to slightly change both \docupdate{ILP}s' runtimes, and as such, we observed no drastic runtime difference as with the fixed clutter dataset. \docupdate{While this comparison was performed with a fixed tool width (grid cell size), changing the tool width could either increase or decrease the number of grid cells for all maps. However, since the change in map area also has the same effect, we expect the observed trend to be maintained with a change in tool width.}

\begin{figure}
    \centering
    \subcaptionbox{}{
    \centering
    \includegraphics[width=0.465\linewidth]{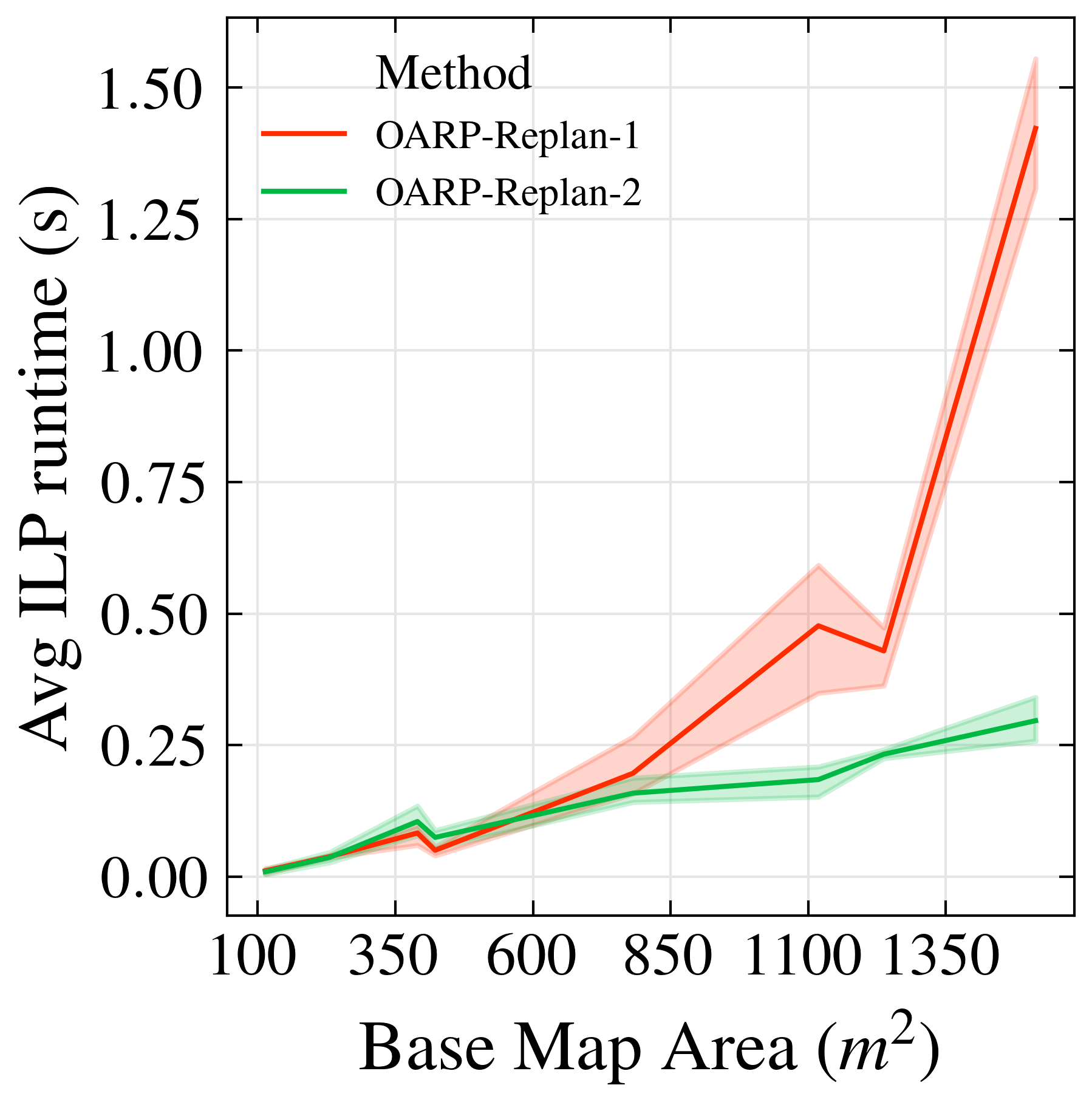}
    }
    \subcaptionbox{}{
    \centering
    \includegraphics[width=0.465\linewidth]{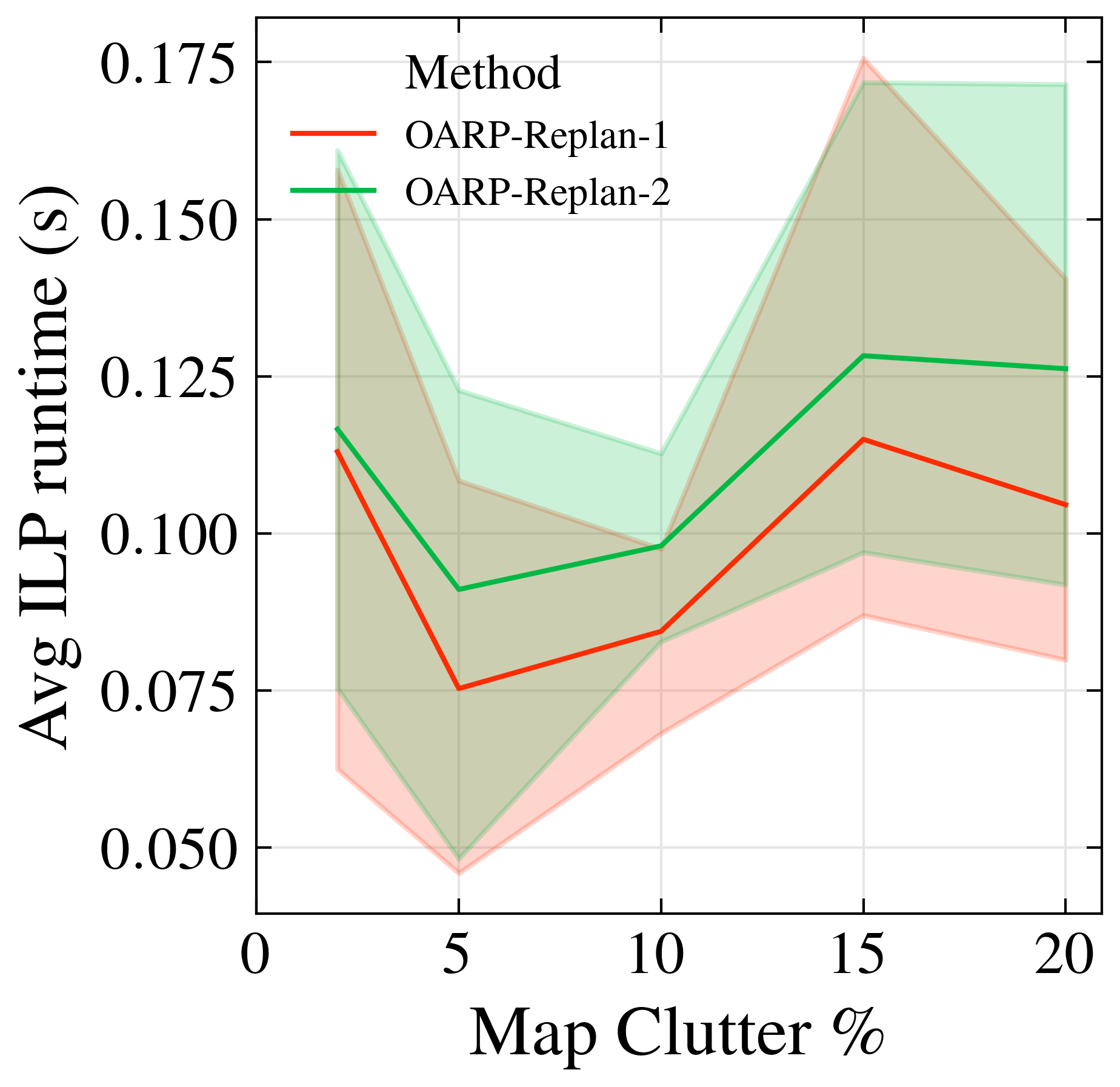}
    }
     \caption{Comparison of average \docupdate{ILP} runtimes for the OARP-Replan approaches. (a) Fixed clutter dataset comparisons, with maps sorted by \textit{base map area}. (b) Varying clutter dataset comparisons. The line shows the average and the lighter region shows the data interval.}
    \label{fig:comparison-results-runtime}
\end{figure}

\subsection{Comparisons with offline approaches}
\label{sec:results-BCD}

\docupdate{In this section, we compare the proposed OARP-Replan approach (i.e. OARP-Replan-1) against \textit{offline} coverage planning approaches, which are provided with all obstacle information prior to coverage. In contrast, we provide OARP-Replan with no prior obstacle information, so it must replan the coverage path whenever the robot detects obstacles. We compare OARP-Replan against 3 offline approaches: (i) boustrophedon cell decomposition (BCD) method \cite{bahnemannRevisitingBoustrophedonCoverage2021}, (ii) line coverage (LC) approach \cite{agarwalAreaCoverageMultiple2022a}, and (iii) the offline OARP planner \cite {rameshOptimalPartitioningNonConvex2022a}. For clarity, we will refer to the offline approaches as \textit{BCD-Offline}, \textit{LC-Offline}, and \textit{OARP-Offline} for the rest of this section. The BCD-Offline algorithm also resembles the BCD approach from the `ipa\_room\_exploration' ROS package, which is the most effective algorithm in the package according to the adjoining survey in \cite{bormannIndoorCoveragePath2018}.}

\docupdate{While OARP-Replan was developed with mobile robots as a focus, it can also be used for coverage planning with aerial robots (fixed coverage tool width). As such, we apply the OARP-Replan approach to plan coverage paths for an aerial robot where the obstacles in the environment are not previously known.} We use the dataset from the BCD paper \cite{bahnemannRevisitingBoustrophedonCoverage2021} that consists of around 300 scans of environments with varying numbers of obstacles or \textit{holes}. The outer boundary of each map is a square, which is then filled with the obstacles observed in the scans. The maps are arranged in the order of increasing complexity, which is given by the number of \textit{hole vertices} (e.g., one square hole consists of 4 vertices). We also use robot parameters similar to that of \docupdate{the UAV from} \cite{bahnemannRevisitingBoustrophedonCoverage2021}: a tool width of 3 $m$, maximum velocity of 3 $m/s$, and linear acceleration of $\pm 1$ $m/s^2$ (\docupdate{turning time omitted to match \cite{bahnemannRevisitingBoustrophedonCoverage2021}}). \docupdate{Additionally, OARP-Replan is provided with an initial path for the unobstructed base environment (a square), and a radial range sensor similar to previous sections to identify obstacles.}

\begin{figure}
    \centering
    \subcaptionbox{}
    {\centering
    \includegraphics[width=0.68\linewidth]{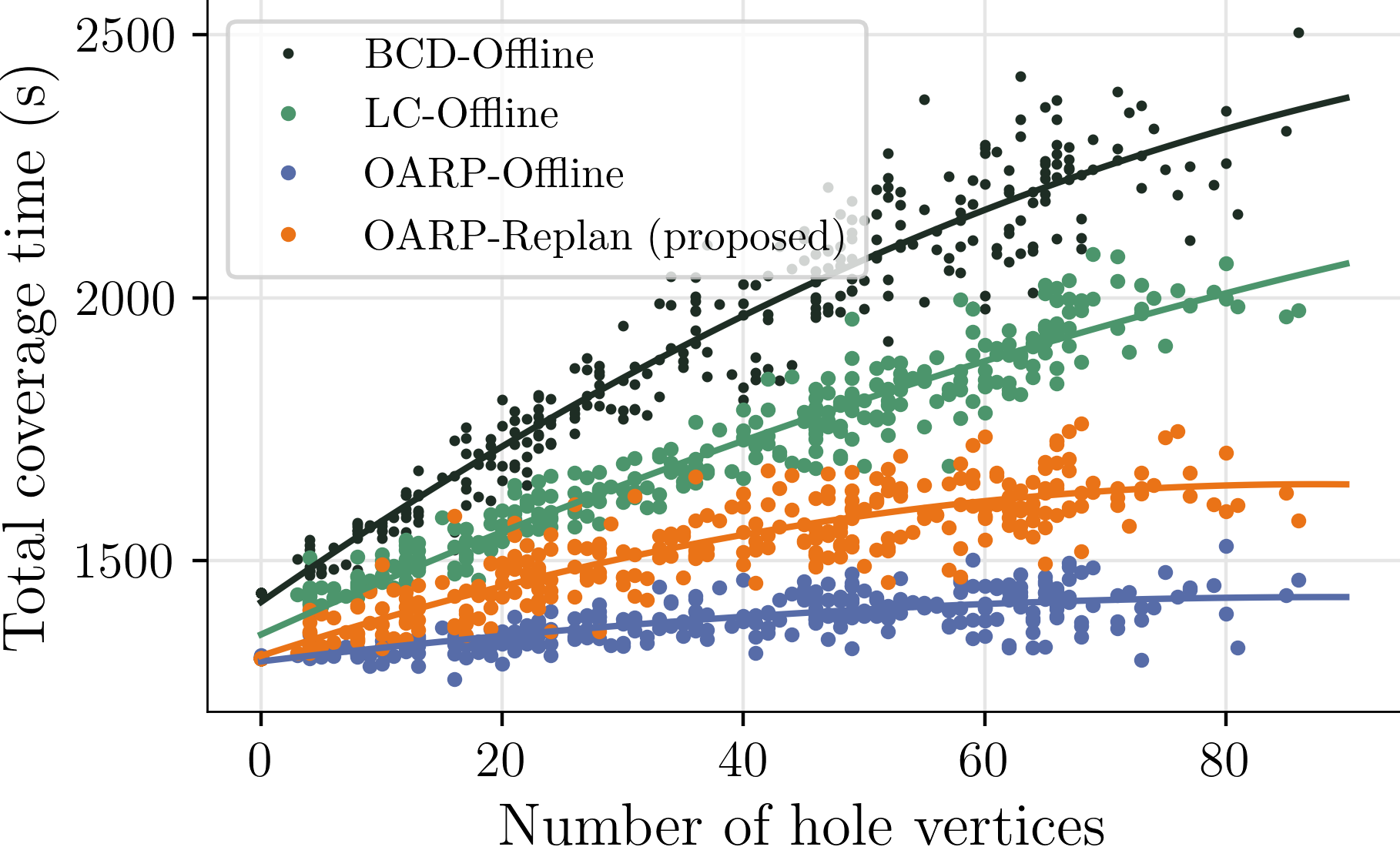}
    }
    \subcaptionbox{}
    {\centering
    \includegraphics[width=0.28\linewidth]{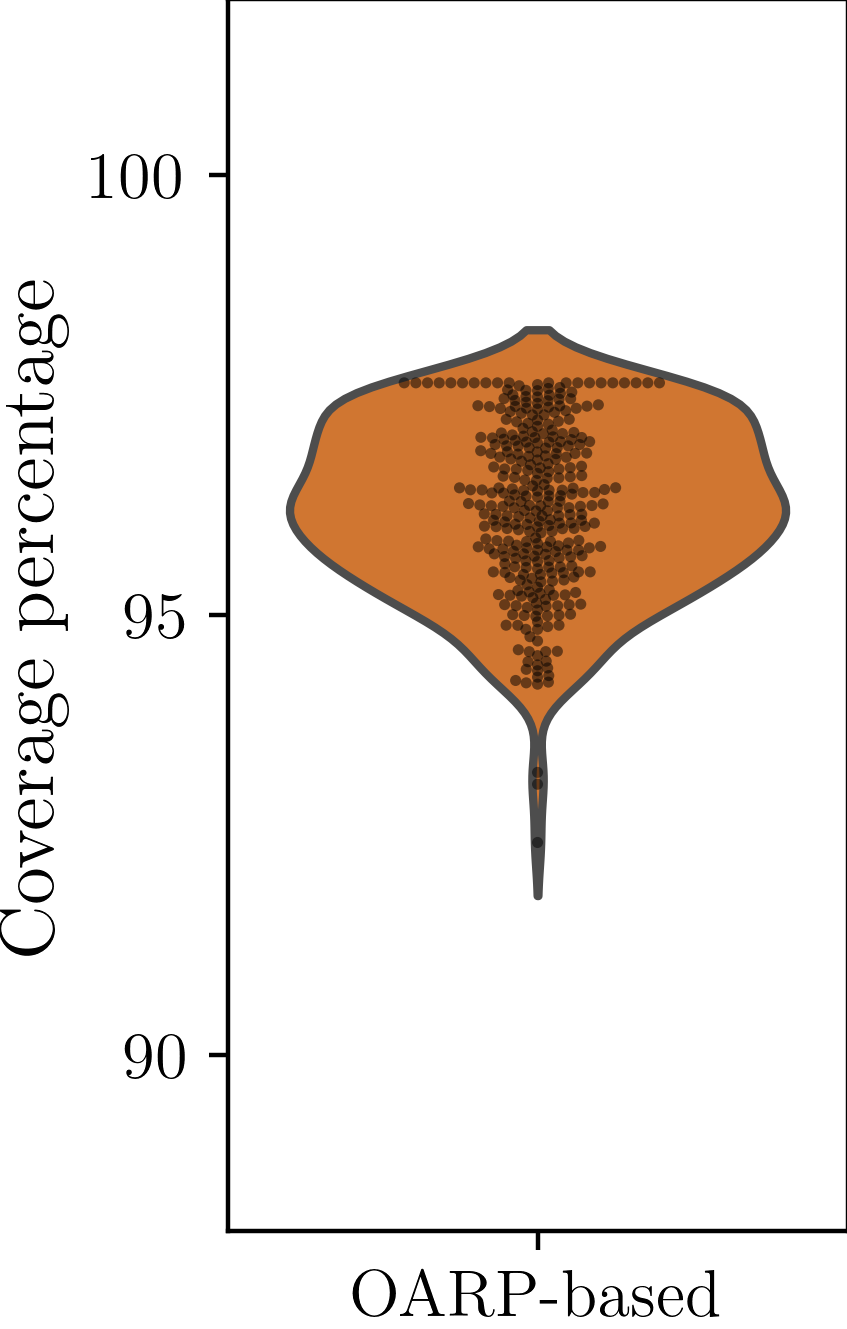}
    }
    \caption{\docupdate{Comparison results with offline planning approaches. (a) Coverage time comparisons between OARP-Replan and offline baselines. (b) Violin plot of percent coverage achieved by the OARP-based approaches (with data points plotted). BCD-Offline and LC-Offline are omitted as they are exact approaches that achieve 100\% coverage.}
    }
    \label{fig:bcd-comparison}
\end{figure}

Fig. \ref{fig:bcd-comparison}(a) shows the results of comparing the total coverage times of using all approaches. We observe that OARP-Replan outperforms BCD-Offline and LC-Offline across all maps, with an improvement of about \docupdate{25\%} on average from BCD-Offline. We believe this improvement to be due to the optimal coverage lines produced by both OARP-based methods, as BCD-Offline paths tend to overlap closer to cell boundaries and obstacles. We also observe that OARP-Offline exhibits the best performance with an average improvement of \docupdate{31\%} in coverage time from BCD-Offline. \docupdate{However, note that OARP-Offline is equivalent to OARP-Replan when provided with all obstacle information a-priori and unlimited solver time. We conclude that using OARP-Replan with a simple initial path is a competitive alternative to using BCD-Offline with complete obstacle information.} However, if obstacle information is available beforehand, OARP-Offline is a better alternative as it computes paths with the fastest coverage times. 

\docupdate{One thing to note is that both BCD-Offline and LC-Offline are exact approaches that aim to cover the entire environment (100\% coverage). In contrast, the OARP-based approaches use grid approximations that may not cover the environment fully. To address this, we plot the percent coverage achieved by both OARP-based methods (they cover the same grid) in Fig. \ref{fig:bcd-comparison}(b), where we observe that the OARP-based approaches achieve about 96\% coverage on average across all environments, with missed coverage usually occurring close to obstacles and boundaries. Therefore, we recommend OARP-Replan for covering environments with unknown obstacles where small amounts of missed coverage is acceptable.}

\subsection{\docupdate{Simulation case study}}


This section provides a case study of a Turtlebot 3 replanning coverage paths in a ROS simulation using GD replan and OARP-Replan (i.e. OARP-Replan-1). This is to demonstrate the real-world applicability of our replanning approach \docupdate{in contrast to GD replanning}. We generated an anonymized environment as shown in Fig. \ref{fig:ros-simulation}, for which an initial coverage path was planned (Fig. \ref{fig:ros-simulation}(a)). We added new obstacles to the environment that occupied 10\% of the initial area. As the robot observes new obstacles blocking this path, the robot replans the path either using OARP-Replan or GD replan. The replanning is done while driving and the robot switches to the new path when available. The robot is also equipped with a local planner to traverse the path and make small changes if the path goes close to the walls. Fig. \ref{fig:ros-simulation}(b) shows a snapshot of the robot replanning the coverage path using OARP-Replan. We use the robot parameters from Table \ref{table:parameters} for this simulation with some changes to use the default Turtlebot controller in ROS: we set maximum velocity to 0.3 $m/s$ and linear acceleration to 1 $m/s^2$. With these parameters, we observed that OARP-Replan covered the environment 16.5\% \docupdate{faster} than GD replan along a path that was 12.2\% shorter. This difference is mainly due to the reduction in long detours where the robot travels around obstacles many times. We included recordings of this case study in our video attachment.

\begin{figure}
\centering
    \subcaptionbox{}
    {\centering
    \includegraphics[width=0.47\linewidth]{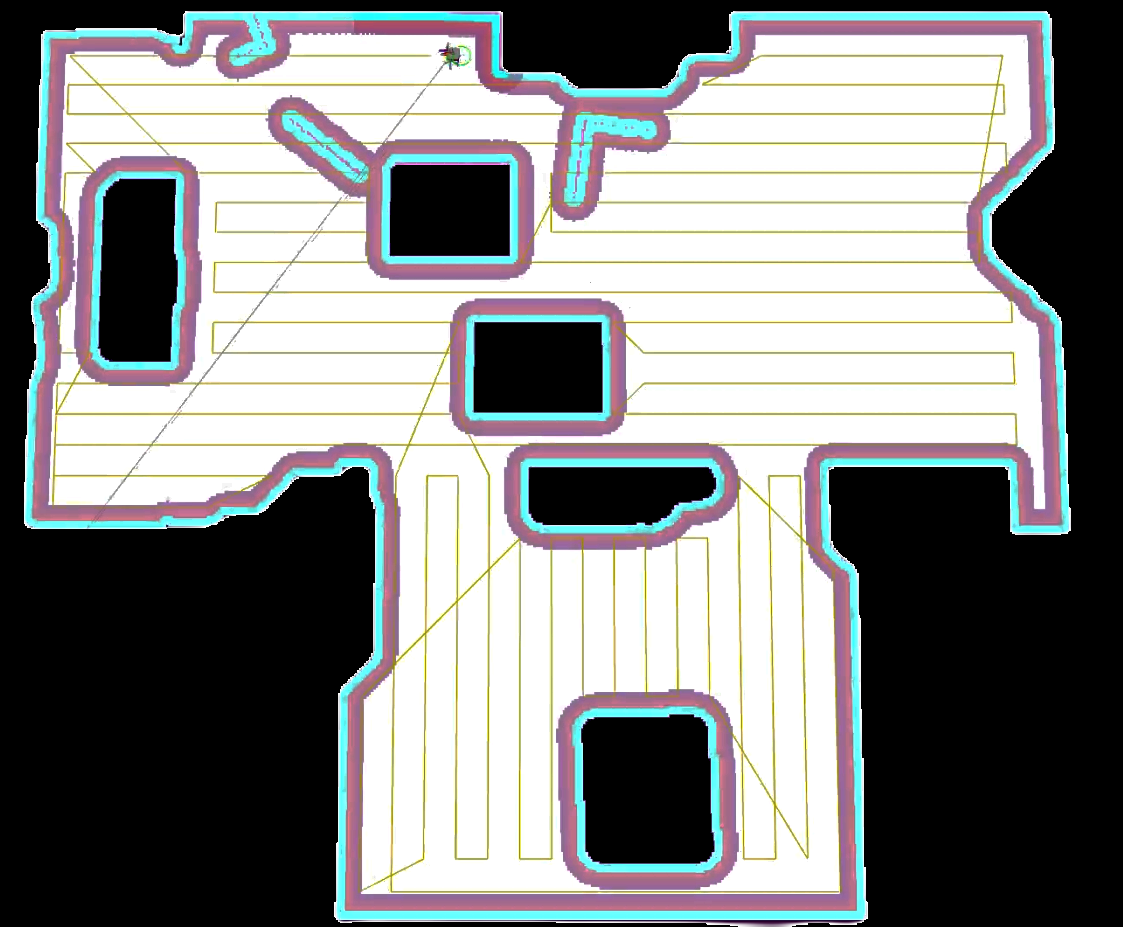}
    }
    \subcaptionbox{}
    {\centering
    \includegraphics[width=0.47\linewidth]{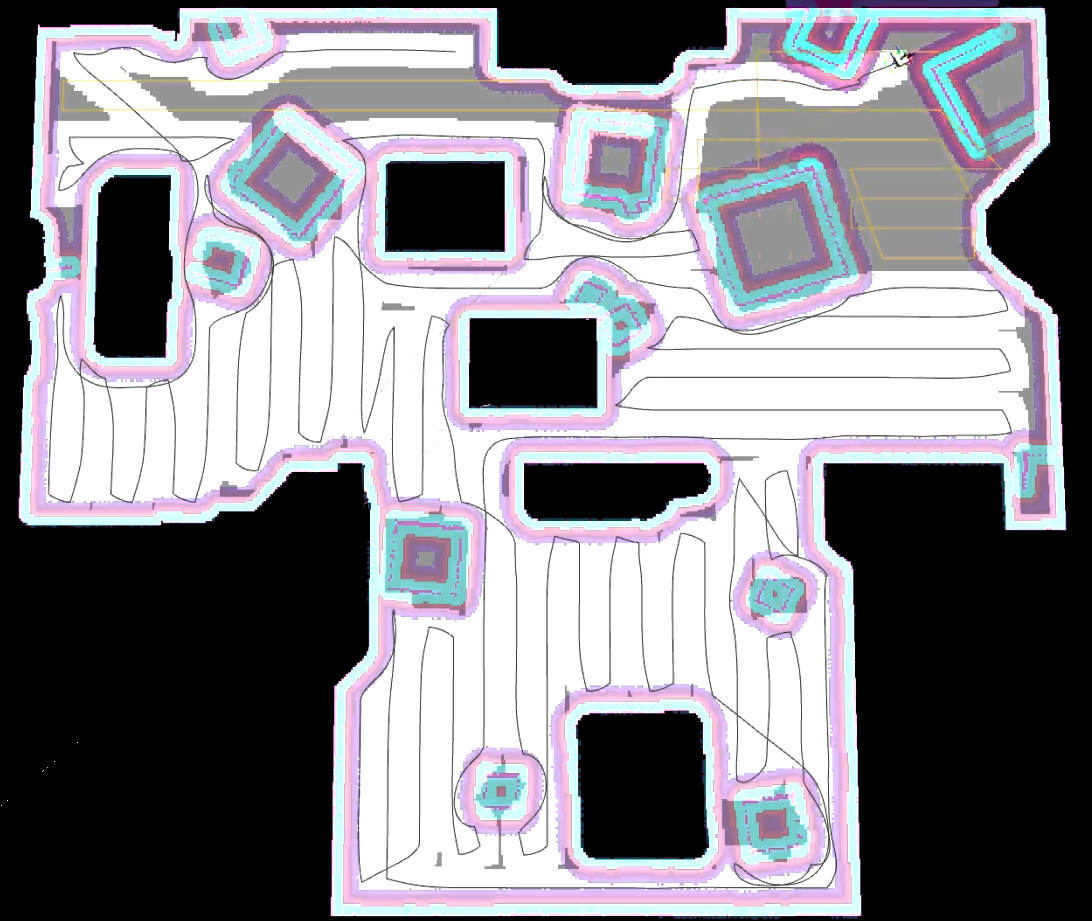}
    }
    \caption{Robot operating system (ROS) simulation of covering an example environment with unknown obstacles. (a) An initial coverage path for the base environment (b) replanning the path online using OARP-Replan.}
    \label{fig:ros-simulation}
\end{figure}

\docupdate{
\section{Robot Experiments}
\label{sec:robot-experiments}

In this section, we present the results of testing OARP-Replan (i.e. OARP-Replan-1) on an Avidbots Neo robot (Fig. \ref{fig:robot-experiment-images}(a)) covering a partially unknown environment. The robot's task is to cover a test environment for which an initial path is provided (Fig. \ref{fig:robot-experiments}(a)). The robot has tricycle dynamics and is equipped with a local planner to make small changes to difficult parts of the initial path. However, the robot may encounter large obstacles during coverage (Fig. \ref{fig:robot-experiment-images}(b)) which may result in narrow regions that the robot cannot reach along its current path. Therefore, a replan of the coverage path is necessary to reach such regions and maximize the coverage of the environment.

We conducted our experiments using two test scenarios. Fig. \ref{fig:robot-experiments} shows the results of applying OARP-Replan in both scenarios. In scenario 1 (Fig. \ref{fig:robot-experiment-images}(a) / Fig. \ref{fig:robot-experiments}(b)), the robot observes an obstacle that interrupts a large number of ranks in the initial path. The path is replanned using OARP-Replan, which computes horizontal ranks to comfortably cover the narrow region (bottom right of the environment in Fig. \ref{fig:robot-experiments}(b)). In scenario 2 (Fig. \ref{fig:robot-experiment-images}(b) / Fig. \ref{fig:robot-experiments}(c)), the robot encounters a large obstacle that it must observe over time. Later in the task, the robot observes a part of the obstacle that causes a narrow coverage region (top right of the environment in Fig. \ref{fig:robot-experiments}(c)). OARP-Replan allows the robot to cover this region by re-orienting the coverage ranks. We provide recordings of these experiments in our multimedia attachment.

To implement OARP-Replan, we require an occupancy map that uses the LIDAR scan collected by the robot to identify grid cells that correspond to static obstacles. However, this occupancy map must differentiate between a static obstacle and a noisy LIDAR scan. We do this by checking consecutive LIDAR readings for obstacles that block grid cells in the occupancy grid and has not moved for some time. While this step increases the time between detecting an obstacle and replanning, it is necessary to prevent faulty replanning of the coverage path as a result of LIDAR scan noise/fluctuations. 

Note that without replanning (i.e. GD replan with the local planner), the coverage path for the narrow regions will have short ranks with turns occurring in quick succession. The Neo robot cannot drive such paths without manual assistance. The proposed replanning approach aims to reduce such instances where manual assistance is required to continue coverage.

\begin{figure}
\centering
\subcaptionbox{Neo robot in Scenario 1}
{\centering
\includegraphics[width=0.46\linewidth]{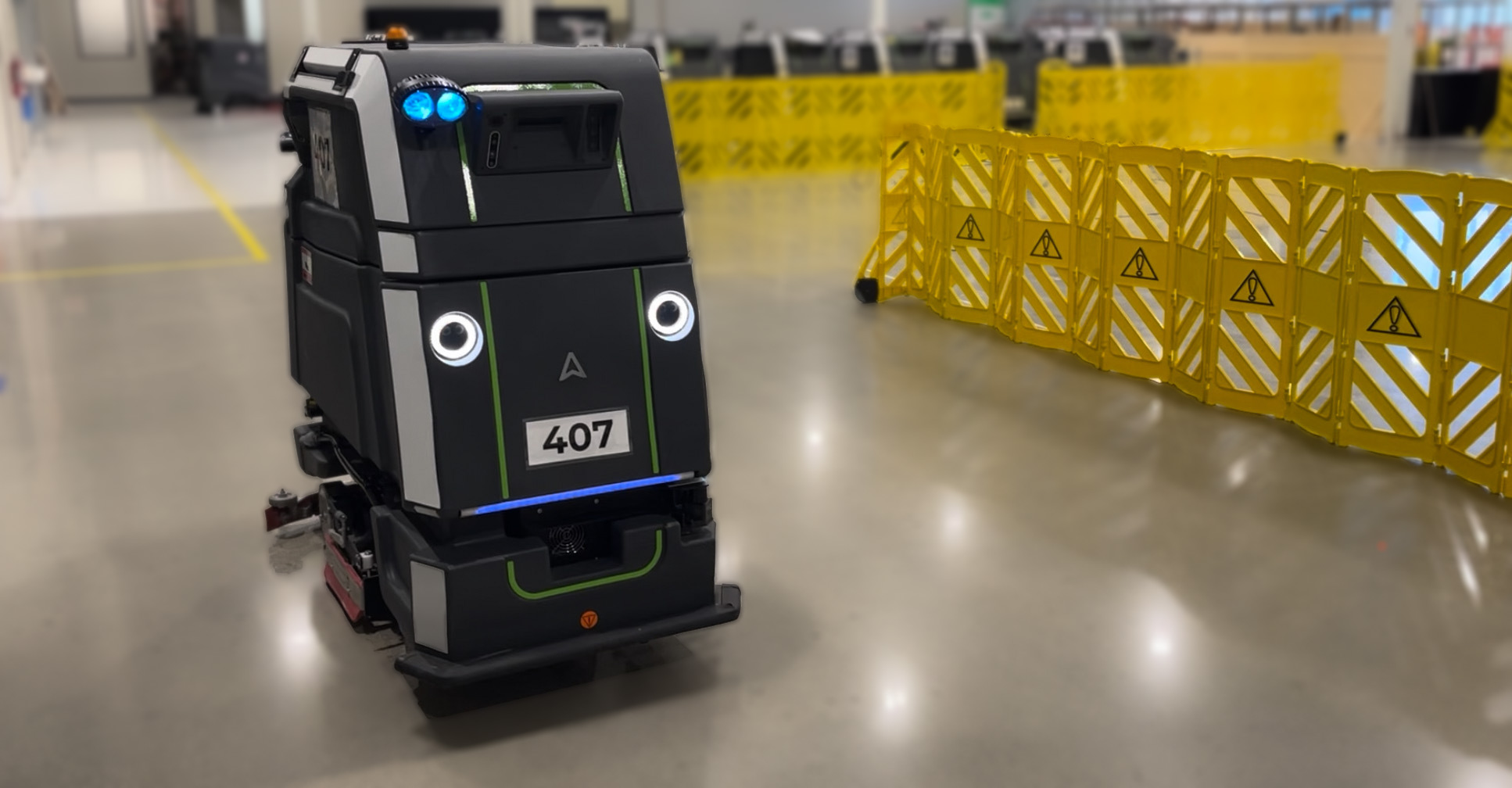}
}
\hspace{2mm}
\subcaptionbox{Scenario 2 obstacles}
{\centering
\includegraphics[width=0.46\linewidth]{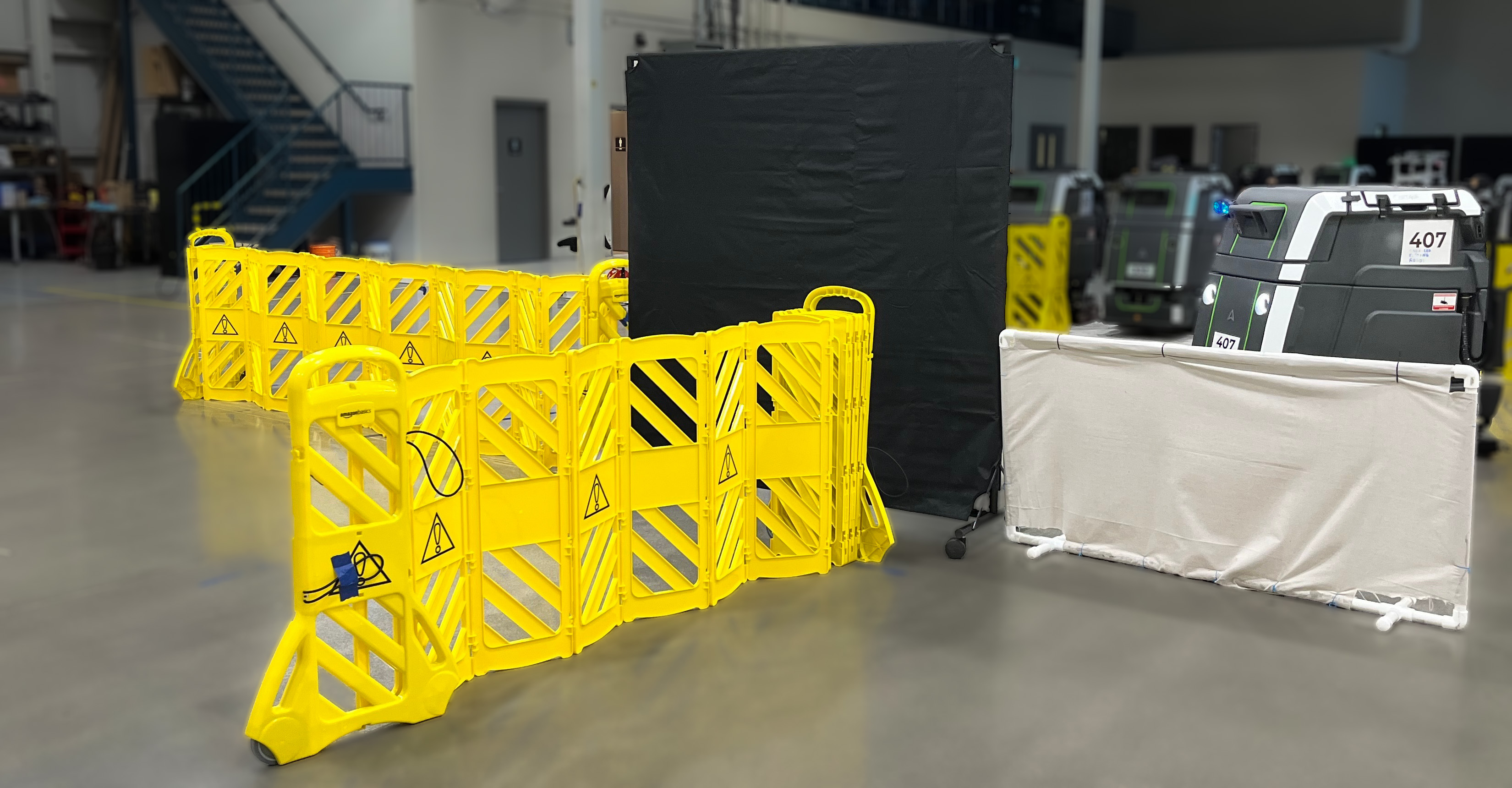}
}
\caption{\docupdate{Snapshots from experiments using the Avidbots Neo robot \vspace{-2mm}}}
\label{fig:robot-experiment-images}
\end{figure}

\begin{figure}
\centering
\subcaptionbox{Initial Path}
{\centering
\includegraphics[width=0.275\linewidth]{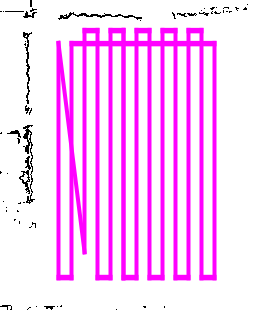}
}
\hspace{2mm}
\subcaptionbox{Scenario 1}
{\centering
\includegraphics[width=0.275\linewidth]{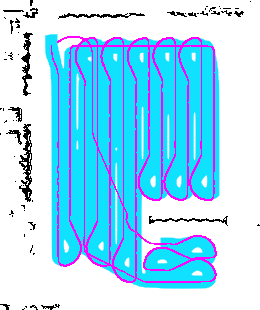}
}
\hspace{2mm}
\subcaptionbox{Scenario 2}
{\centering
\includegraphics[width=0.275\linewidth]{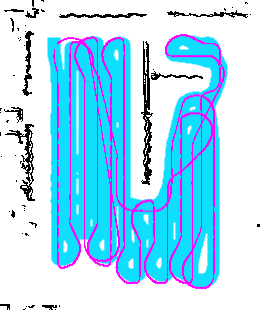}
}
\caption{\docupdate{Robot experiment results of replanning an initial coverage path using OARP-Replan. The final robot path (pink curves) and the covered area (blue) are shown for both scenarios. \vspace{-2mm}}}
\label{fig:robot-experiments}
\end{figure}

}

\section{Conclusion}

We proposed an anytime coverage replanning approach called OARP-Replan to replan coverage paths for environments with unknown static obstacles. OARP-Replan solves this problem in two stages: (i) rank replanning, and (ii) touring. The first stage replans the ranks (straight-line paths) of an obstructed coverage path, while the second stage computes a tour of the replanned ranks to obtain the new path. We proposed two \docupdate{integer linear programs (ILPs)} that each solve the rank replanning stage so that the resulting touring stage is completed within a time budget $\tau$. Simulations on maps of real-world environments showed that the OARP-Replan variants using each \docupdate{ILP} achieve better coverage performance than a greedy replanner and state-of-the-art coverage planners used for online replanning \cite{rameshOptimalPartitioningNonConvex2022a, vandermeulenTurnminimizingMultirobotCoverage2019, kusnurCompleteDecompositionFreeCoverage2022}. Moreover, we also compared OARP-Replan with offline planning approaches (treating all obstacles as known) \cite{ bahnemannRevisitingBoustrophedonCoverage2021, agarwalAreaCoverageMultiple2022a} where we observed an improvement in total coverage time. \docupdate{Finally, we demonstrated the proposed approach by replanning coverage paths for an industrial-level cleaning robot.}

\docupdate{In this work, we focus only on the case of new obstacles in the environment. The case where an obstacle is observed and later removed from the environment (including dynamic obstacles) is left for future work.} This would involve detecting areas in the environment that could be free at a later time due to an obstacle removal and returning to cover them without long detours. Another future direction involves replanning using \textit{predictions} about environment occupancy to balance coverage and exploration, which are competing objectives. \docupdate{To balance them, the coverage planner must determine \textit{when} to replan the coverage path, perhaps even delaying a replan until more information about an obstacle is collected.} Removing the axis-parallel constraint and allowing multiple directions of coverage may also improve coverage paths for non-orthogonal environments.

\section*{Acknowledgment}
This work was supported by Canadian Mitacs Accelerate Project IT16435 and Avidbots Corp, Kitchener, ON, Canada.

\bibliographystyle{IEEEtran}
\bibliography{references, online-ref-new, proposed, mine, song}

\begin{thebibliography}{10}
\providecommand{\url}[1]{#1}
\csname url@samestyle\endcsname
\providecommand{\newblock}{\relax}
\providecommand{\bibinfo}[2]{#2}
\providecommand{\BIBentrySTDinterwordspacing}{\spaceskip=0pt\relax}
\providecommand{\BIBentryALTinterwordstretchfactor}{4}
\providecommand{\BIBentryALTinterwordspacing}{\spaceskip=\fontdimen2\font plus
\BIBentryALTinterwordstretchfactor\fontdimen3\font minus \fontdimen4\font\relax}
\providecommand{\BIBforeignlanguage}[2]{{%
\expandafter\ifx\csname l@#1\endcsname\relax
\typeout{** WARNING: IEEEtran.bst: No hyphenation pattern has been}%
\typeout{** loaded for the language `#1'. Using the pattern for}%
\typeout{** the default language instead.}%
\else
\language=\csname l@#1\endcsname
\fi
#2}}
\providecommand{\BIBdecl}{\relax}
\BIBdecl

\bibitem{hofnerPathPlanningGuidance1995}
C.~Hofner and G.~Schmidt, ``Path planning and guidance techniques for an autonomous mobile cleaning robot,'' \emph{Robotics and Autonomous Systems}, vol.~14, no.~2, pp. 199--212, 1995.

\bibitem{kapanogluPatternbasedGeneticAlgorithm2012}
M.~Kapanoglu, M.~Alikalfa, M.~Ozkan, A.~Yaz{\i}c{\i}, and O.~Parlaktuna, ``A pattern-based genetic algorithm for multi-robot coverage path planning minimizing completion time,'' \emph{Journal of Intelligent Manufacturing}, vol.~23, no.~4, pp. 1035--1045, Aug. 2012.

\bibitem{hameedIntelligentCoveragePath2014}
I.~A. Hameed, ``Intelligent coverage path planning for agricultural robots and autonomous machines on three-dimensional terrain,'' \emph{Journal of Intelligent and Robotic Systems: Theory and Applications}, vol.~74, no. 3-4, pp. 965--983, 2014.

\bibitem{galceranSurveyCoveragePath2013}
E.~Galceran and M.~Carreras, ``A survey on coverage path planning for robotics,'' \emph{Robotics and Autonomous Systems}, vol.~61, no.~12, pp. 1258--1276, 2013.

\bibitem{arkinApproximationAlgorithmsLawn2000}
E.~M. Arkin, S.~P. Fekete, and J.~S. Mitchell, ``Approximation algorithms for lawn mowing and milling,'' \emph{Computational Geometry}, vol.~17, no. 1-2, pp. 25--50, 2000.

\bibitem{rameshOptimalPartitioningNonConvex2022a}
M.~Ramesh, F.~Imeson, B.~Fidan, and S.~L. Smith, ``Optimal {{Partitioning}} of {{Non-Convex Environments}} for {{Minimum Turn Coverage Planning}},'' \emph{IEEE Robotics and Automation Letters}, vol.~7, no.~4, pp. 9731--9738, Oct. 2022.

\bibitem{bahnemannRevisitingBoustrophedonCoverage2021}
R.~B{\"a}hnemann, N.~Lawrance, J.~J. Chung, M.~Pantic, R.~Siegwart, and J.~Nieto, ``Revisiting {{Boustrophedon Coverage Path Planning}} as a {{Generalized Traveling Salesman Problem}},'' in \emph{Field and {{Service Robotics}}}, ser. Springer {{Proceedings}} in {{Advanced Robotics}}.\hskip 1em plus 0.5em minus 0.4em\relax {Singapore}: {Springer}, 2021, pp. 277--290.

\bibitem{agarwalAreaCoverageMultiple2022a}
S.~Agarwal and S.~Akella, ``Area {{Coverage With Multiple Capacity-Constrained Robots}},'' \emph{IEEE Robotics and Automation Letters}, vol.~7, no.~2, pp. 3734--3741, Apr. 2022.

\bibitem{bormannIndoorCoveragePath2018}
R.~Bormann, F.~Jordan, J.~Hampp, and M.~H{\"a}gele, ``Indoor {{Coverage Path Planning}}: {{Survey}}, {{Implementation}}, {{Analysis}},'' in \emph{2018 {{IEEE International Conference}} on {{Robotics}} and {{Automation}} ({{ICRA}})}, May 2018, pp. 1718--1725.

\bibitem{tanComprehensiveReviewCoverage2021}
C.~S. Tan, R.~{Mohd-Mokhtar}, and M.~R. Arshad, ``A {{Comprehensive Review}} of {{Coverage Path Planning}} in {{Robotics Using Classical}} and {{Heuristic Algorithms}},'' \emph{IEEE Access}, vol.~9, pp. 119\,310--119\,342, 2021.

\bibitem{chosetCoveragePathPlanning1998}
H.~Choset and P.~Pignon, ``Coverage {{Path Planning}}: {{The Boustrophedon Cellular Decomposition}},'' \emph{Field and Service Robotics}, pp. 203--209, 1998.

\bibitem{cabreiraEnergyAwareSpiralCoverage2018}
T.~M. Cabreira, C.~D. Franco, P.~R. Ferreira, and G.~C. Buttazzo, ``Energy-{{Aware Spiral Coverage Path Planning}} for {{UAV Photogrammetric Applications}},'' \emph{IEEE Robotics and Automation Letters}, vol.~3, no.~4, pp. 3662--3668, Oct. 2018.

\bibitem{brownConstrictionDecompositionMethod2016}
S.~Brown and S.~L. Waslander, ``The constriction decomposition method for coverage path planning,'' in \emph{2016 {{IEEE}}/{{RSJ International Conference}} on {{Intelligent Robots}} and {{Systems}} ({{IROS}})}, Oct. 2016, pp. 3233--3238.

\bibitem{bochkarevMinimizingTurnsRobot2016}
S.~Bochkarev and S.~L. Smith, ``On minimizing turns in robot coverage path planning,'' in \emph{2016 {{IEEE International Conference}} on {{Automation Science}} and {{Engineering}} ({{CASE}})}, 2016, pp. 1237--1242.

\bibitem{dasMappingPlanningSample2014}
A.~Das, M.~Diu, N.~Mathew, C.~Scharfenberger, J.~Servos, A.~Wong, J.~S. Zelek, D.~A. Clausi, and S.~L. Waslander, ``Mapping, {{Planning}}, and {{Sample Detection Strategies}} for {{Autonomous Exploration}},'' \emph{Journal of Field Robotics}, vol.~31, no.~1, pp. 75--106, 2014.

\bibitem{gabrielySpanningtreeBasedCoverage2001}
Y.~Gabriely and E.~Rimon, ``Spanning-tree based coverage of continuous areas by a mobile robot,'' \emph{Annals of Mathematics and Artificial Intelligence}, vol.~31, no. 1-4, pp. 77--98, 2001.

\bibitem{zelinskyUnifiedApproachPlanning1994}
A.~Zelinsky and S.~Yuta, ``A unified approach to planning, sensing and navigation for mobile robots,'' in \emph{Experimental {{Robotics III}}}, ser. Lecture {{Notes}} in {{Control}} and {{Information Sciences}}.\hskip 1em plus 0.5em minus 0.4em\relax Berlin, Heidelberg: Springer, 1994, pp. 444--455.

\bibitem{vandermeulenTurnminimizingMultirobotCoverage2019}
I.~Vandermeulen, R.~Gro{\ss}, and A.~Kolling, ``Turn-minimizing multirobot coverage,'' in \emph{2019 {{IEEE International Conference}} on {{Robotics}} and {{Automation}} ({{ICRA}})}, 2019, pp. 1014--1020.

\bibitem{luTMSTCPathPlanning2023}
J.~Lu, B.~Zeng, J.~Tang, T.~L. Lam, and J.~Wen, ``{{TMSTC}}*: {{A Path Planning Algorithm}} for {{Minimizing Turns}} in {{Multi-Robot Coverage}},'' \emph{IEEE Robotics and Automation Letters}, vol.~8, no.~8, pp. 5275--5282, Aug. 2023.

\bibitem{acarSensorbasedCoverageUnknown2002}
E.~U. Acar and H.~Choset, ``Sensor-based {{Coverage}} of {{Unknown Environments}}: {{Incremental Construction}} of {{Morse Decompositions}},'' \emph{The International Journal of Robotics Research}, vol.~21, no.~4, pp. 345--366, Apr. 2002.

\bibitem{vietBAOnlineComplete2013}
H.~H. Viet, V.-H. Dang, M.~N.~U. Laskar, and T.~Chung, ``{{BA}}*: An online complete coverage algorithm for cleaning robots,'' \emph{Applied Intelligence}, vol.~39, no.~2, pp. 217--235, Sep. 2013.

\bibitem{gabrielyCompetitiveOnlineCoverage2003}
Y.~Gabriely and E.~Rimon, ``Competitive on-line coverage of grid environments by a mobile robot,'' \emph{Computational Geometry}, vol.~24, no.~3, pp. 197--224, 2003.

\bibitem{song2018}
J.~Song and S.~Gupta, ``$\varepsilon ^{\star }$: An online coverage path planning algorithm,'' \emph{IEEE Transactions on Robotics}, vol.~34, no.~2, pp. 526--533, 2018.

\bibitem{luoBioinspiredNeuralNetwork2008}
C.~Luo and S.~X. Yang, ``A {{Bioinspired Neural Network}} for {{Real-Time Concurrent Map Building}} and {{Complete Coverage Robot Navigation}} in {{Unknown Environments}},'' \emph{IEEE Transactions on Neural Networks}, vol.~19, no.~7, pp. 1279--1298, Jul. 2008.

\bibitem{muthugalaEnergyefficientOnlineComplete2022}
M.~A. V.~J. Muthugala, S.~M. B.~P. Samarakoon, and M.~R. Elara, ``Toward energy-efficient online {{Complete Coverage Path Planning}} of a ship hull maintenance robot based on {{Glasius Bio-inspired Neural Network}},'' \emph{Expert Systems with Applications}, vol. 187, p. 115940, Jan. 2022.

\bibitem{sunCompleteCoverageAutonomous2019a}
B.~Sun, D.~Zhu, C.~Tian, and C.~Luo, ``Complete {{Coverage Autonomous Underwater Vehicles Path Planning Based}} on {{Glasius Bio-Inspired Neural Network Algorithm}} for {{Discrete}} and {{Centralized Programming}},'' \emph{IEEE Transactions on Cognitive and Developmental Systems}, vol.~11, no.~1, pp. 73--84, Mar. 2019.

\bibitem{godioBioinspiredNeuralNetworkBased2021}
S.~Godio, S.~Primatesta, G.~Guglieri, and F.~Dovis, ``A {{Bioinspired Neural Network-Based Approach}} for {{Cooperative Coverage Planning}} of {{UAVs}},'' \emph{Information}, vol.~12, no.~2, p.~51, Feb. 2021.

\bibitem{dangGraphbasedSubterraneanExploration2020a}
T.~Dang, M.~Tranzatto, S.~Khattak, F.~Mascarich, K.~Alexis, and M.~Hutter, ``Graph-based subterranean exploration path planning using aerial and legged robots,'' \emph{Journal of Field Robotics}, vol.~37, no.~8, pp. 1363--1388, 2020.

\bibitem{corahEfficientOnlineMultirobot2017}
M.~Corah and N.~Michael, ``Efficient {{Online Multi-robot Exploration}} via {{Distributed Sequential Greedy Assignment}},'' in \emph{Robotics: {{Science}} and {{Systems XIII}}}.\hskip 1em plus 0.5em minus 0.4em\relax {Robotics: Science and Systems Foundation}, Jul. 2017.

\bibitem{huOffPolicyEvaluationOnline2023}
Y.~Hu, J.~Geng, C.~Wang, J.~Keller, and S.~Scherer, ``Off-{{Policy Evaluation With Online Adaptation}} for {{Robot Exploration}} in {{Challenging Environments}},'' \emph{IEEE Robotics and Automation Letters}, vol.~8, no.~6, pp. 3780--3787, Jun. 2023.

\bibitem{yamauchiFrontierbasedApproachAutonomous1997}
B.~Yamauchi, ``A frontier-based approach for autonomous exploration,'' in \emph{Proceedings 1997 {{IEEE International Symposium}} on {{Computational Intelligence}} in {{Robotics}} and {{Automation CIRA}}'97. '{{Towards New Computational Principles}} for {{Robotics}} and {{Automation}}'}.\hskip 1em plus 0.5em minus 0.4em\relax Monterey, CA, USA: IEEE Comput. Soc. Press, 1997, pp. 146--151.

\bibitem{kusnurCompleteDecompositionFreeCoverage2022}
T.~Kusnur and M.~Likhachev, ``Complete, {{Decomposition-Free Coverage Path Planning}},'' in \emph{2022 {{IEEE}} 18th {{International Conference}} on {{Automation Science}} and {{Engineering}} ({{CASE}})}, Aug. 2022, pp. 1431--1437.

\bibitem{galceranCoveragePathPlanning2015}
E.~Galceran, R.~Campos, N.~Palomeras, D.~Ribas, M.~Carreras, and P.~Ridao, ``Coverage {{Path Planning}} with {{Real-time Replanning}} and {{Surface Reconstruction}} for {{Inspection}} of {{Three-dimensional Underwater Structures}} using {{Autonomous Underwater Vehicles}}: {{Coverage Path Planning}} with {{Real-time Replanning}} and {{Surface Reconstruction}},'' \emph{Journal of Field Robotics}, vol.~32, no.~7, pp. 952--983, Oct. 2015.

\bibitem{arkinOptimalCoveringTours2005}
E.~M. Arkin, M.~A. Bender, E.~D. Demaine, S.~P. Fekete, J.~S.~B. Mitchell, and S.~Sethia, ``Optimal {{Covering Tours}} with {{Turn Costs}},'' \emph{SIAM Journal on Computing}, vol.~35, no.~3, pp. 531--566, 2005.

\bibitem{korteCombinatorialOptimizationTheory2006}
B.~Korte and J.~Vygen, \emph{Combinatorial {{Optimization}}: {{Theory}} and {{Algorithms}}}, 3rd~ed., ser. Algorithms and {{Combinatorics}}.\hskip 1em plus 0.5em minus 0.4em\relax {Berlin Heidelberg}: {Springer-Verlag}, 2006.

\bibitem{karmarkarNewPolynomialtimeAlgorithm1984}
N.~Karmarkar, ``A new polynomial-time algorithm for linear programming,'' \emph{Combinatorica}, vol.~4, no.~4, pp. 373--395, 1984.

\bibitem{sahaPlanningToursRobotic2006}
M.~Saha, T.~Roughgarden, J.-C. Latombe, and G.~{S{\'a}nchez-Ante}, ``Planning {{Tours}} of {{Robotic Arms}} among {{Partitioned Goals}},'' \emph{The International Journal of Robotics Research}, vol.~25, no.~3, pp. 207--223, Mar. 2006.

\bibitem{achterbergSCIPSolvingConstraint2009}
T.~Achterberg, ``{{SCIP}}: Solving constraint integer programs,'' \emph{Mathematical Programming Computation}, vol.~1, no.~1, pp. 1--41, Jul. 2009.

\bibitem{smithGLNSEffectiveLarge2017}
S.~L. Smith and F.~Imeson, ``{{GLNS}}: {{An}} effective large neighborhood search heuristic for the {{Generalized Traveling Salesman Problem}},'' \emph{Computers \& Operations Research}, vol.~87, pp. 1--19, 2017.

\end{thebibliography}

\end{document}